\documentclass{article}

\PassOptionsToPackage{numbers, compress, sort}{natbib}

\usepackage[preprint]{neurips_2026}

\usepackage[utf8]{inputenc} 
\usepackage[T1]{fontenc}    
\usepackage{hyperref}       
\usepackage{url}            
\usepackage{booktabs}       
\usepackage{amsfonts}       
\usepackage{nicefrac}       
\usepackage{microtype}      
\usepackage{xcolor}         

\title{NOFE -- Neural Operator Function Embedding}

\author{%
  Lars Uebbing \\
  UiT The Arctic University of Norway \\
  \texttt{lars.uebbing@uit.no} \\
  \And
  Siyan Chen \\
  UiT The Arctic University of Norway \\
  \texttt{siyan.chen@uit.no} \\
  \And
  Harald L. Joakimsen \\
  UiT The Arctic University of Norway \\
  \texttt{harald.l.joakimsen@uit.no} \\
  \And
  Georgios Leontidis \\
  UiT The Arctic University of Norway \\
  \texttt{georgios.leontidis@uit.no} \\
  \And
  Michael C. Kampffmeyer \\
  UiT The Arctic University of Norway \\
  Norwegian Computing Center \\
  \texttt{michael.c.kampffmeyer@uit.no} \\
  \And
  Kristoffer K. Wickstrøm \\
  UiT The Arctic University of Norway \\
  \texttt{kristoffer.k.wickstrom@uit.no} \\
  \And
  Sébastien Lefèvre \\
  University of South Brittany \\
  UiT The Arctic University of Norway \\
  \texttt{sebastien.lefevre@irisa.fr} \\
  \And
  Arnt-Børre Salberg \\
  Norwegian Computing Center \\
  \texttt{salberg@nr.no} \\
  \And
  Robert Jenssen \\
  UiT The Arctic University of Norway \\
  University of Copenhagen \\
  Norwegian Computing Center \\
  \texttt{robert.jenssen@uit.no} \\
}

\usepackage{graphicx} 
\usepackage{amsmath}
\usepackage{amssymb}
\usepackage{bbold}
\usepackage{xcolor}
\usepackage{amsthm}
\usepackage{comment}
\usepackage{multirow}
\usepackage{algorithm}
\usepackage{algpseudocode}

\usepackage{subcaption}
\usepackage{wrapfig}
\usepackage{array} 

\usepackage{multirow}
\usepackage{makecell}

\newtheorem{definition}{Definition}
\newtheorem{proposition}{Proposition}

\newtheorem{assumption}{Assumption}

\usepackage[textsize=tiny]{todonotes}

\newcommand{\Robert}[2][]{\todo[color=purple!30, author=Robert, #1]{#2}}
\newcommand{\Georgios}[2][]{\todo[color=blue!30, author=Georgios, #1]{#2}}

\begin{document}

\maketitle

\begin{abstract}

Most dimensionality reduction methods treat data as discrete point clouds, ignoring the continuous domain structure inherent to many real-world processes. To bridge this gap, we introduce Neural Operator Function Embedding (NOFE), a domain-aware framework for continuous dimensionality reduction. NOFE learns function-to-function mappings via a Graph Kernel Operator, enabling mesh-free evaluation at arbitrary query locations independent of input discretization. We establish NOFE as approximation of sheaf-to-sheaf mappings, generalizing Sheaf Neural Networks to continuous domains.
We evaluate NOFE across different datasets, comparing it against six traditional as well as recently developed methods, including PCA, t-SNE, UMAP. Our results demonstrate that NOFE significantly outperforms baselines in local structure preservation, achieving a local Stress of 0.125 compared to 0.464 for PCA, 0.694 $-$ 2.127 for remaining methods on the ERA5 climate reanalysis dataset. NOFE also exhibits robust sampling independence, reducing the Patch Stitching Error by up to $20.0\times$ relative to UMAP (13.035 vs. 267.586 under regional normalization) and ensuring consistency across disjoint domain patches. While maintaining competitive global structure preservation (Stress-1: 0.379 vs. PCA's 0.268), NOFE resolves fine-grained structures and produces smooth, consistent embeddings that generalize across varying sample densities, addressing key limitations of discrete reduction methods.

\end{abstract}

\section{Introduction}

\begin{wrapfigure}{r}{0.45\textwidth}
	\centering
	\includegraphics[width=1.05\linewidth,
	 ]{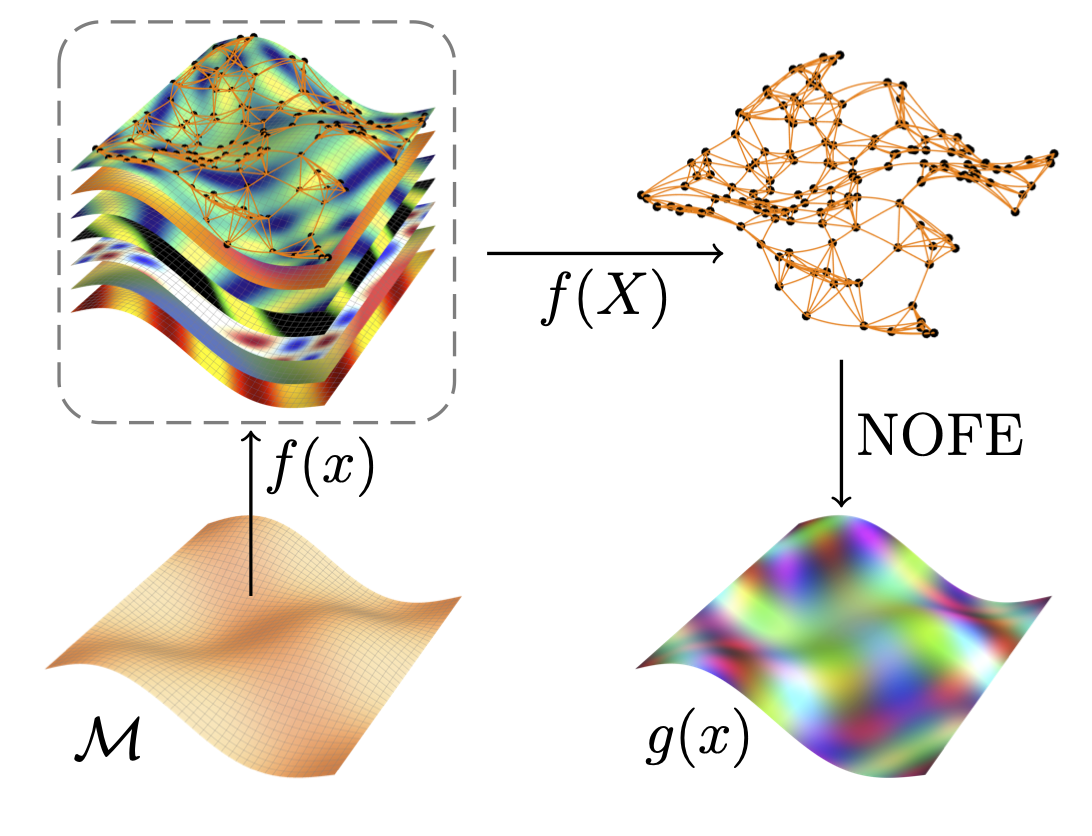}
	\caption{NOFE scheme. For a high-dimensional function $f:\mathcal{M}\rightarrow \mathbb{R}^{d_f}$, a subset of points $X\subset\mathcal{M}$ and their function values $f(X)$ are used to construct a graph, which NOFE maps to a lower dimensional function defined over the same domain $g:\mathcal{M}\rightarrow \mathbb{R}^{d_g}$}
	\label{fig: NOFE scheme}	
    \vspace{-15pt}
\end{wrapfigure}
In modern data analysis, datasets are increasingly high-dimensional, containing hundreds or even thousands of features. While such data is rich in information, it introduces significant challenges, including high computational costs, difficulties in visualization, and the risk of overfitting in predictive models. Dimensionality reduction addresses these challenges by transforming data into a lower-dimensional representation that preserves essential structure and patterns, enabling more efficient computation and improved model interpretability \cite{ghojogh2023ElementsDimensionalityReductiona, waggoner2021ModernDimensionReduction}. 

Often, the dataset $D=\{y_i\}$ analyzed consists of discrete measurements or samples targeting a continuous underlying process $f:\mathcal{M} \rightarrow \mathbb{R}^{d_f}$ over some domain $\mathcal{M}$, with $y_i = f(x_i)$. These discretizations of continuous data structures could be anything from a time series, pixels of an image, spatial measurements of observables or more. Existing methods, however, rarely account for the continuous nature of the data, nor the domain in which the data lives. Much used techniques like Principal Component Analysis (PCA) \cite{2002PrincipalComponentAnalysis}, t-distributed Stochastic Neighborhood Embedding (t-SNE) \cite{maaten2008VisualizingDataUsing}, Uniform Manifold Approximation and Projection (UMAP) \cite{mcinnes2020UMAPUniformManifold}, Potential of Heat-diffusion for Affinity-based Trajectory Embedding (PHATE) \cite{moon2019VisualizingStructureTransitions} or more recent approaches such as Contrastive Learning for Cross-domain Reconciliation and Discovery (CONCORD) \cite{zhu2026RevealingCoherentCellstate} treat the data purely as a set of discrete points $y_i$, ignoring the smooth functional relationships inherent in the underlying system. This can lead to overclustering, exaggeration of distances between points, and misrepresentation of the true geometry of the data, producing embeddings that distort or obscure meaningful patterns \cite{liu2025AssessingImprovingReliability, bergam2026TSNEExaggeratesClusters}. While methods such as SpatialPCA additionally incorporate spatial information, they remain formulated on sampled observations rather than as mapping between continuous functions \cite{shang2022SpatiallyAwareDimension}.

With the recently emerging concept of Neural Operators (NOs) \cite{li2020NeuralOperatorGraph, lu2021DeepONetLearningNonlinear, li2021FourierNeuralOperator, kovachki2023NeuralOperatorLearning}, we overcome the above mentioned limitations and develop an approach for dimensionality reduction  that is fundamentally different from existing methods as it aims to find embeddings for the underlying function rather than the mere set of observations. We propose Neural Operator Function Embedding (NOFE) --- a formalism to learn operators for continuous, domain aware dimensionality reduction. NOFE uses a Graph Kernel Operator (GKO) \cite{li2020NeuralOperatorGraph}, where the domain $\mathcal{M}$ provides the graph structure and the function $f(x)$ yields node features (see Fig. \ref{fig: NOFE scheme}). This way, the framework can generate a continuous output function $g(x)$ regardless of the input samples.

By reformulating the problem as a sheaf morphism \cite{curry2014SheavesCosheavesApplications}, we build on a rich mathematical foundation, opening new perspectives on dimensionality reduction. We show that GKOs, under constraints of locality, can be seen as an extension of existing Sheaf Neural Networks (SNNs) \cite{hansen2020SheafNeuralNetworks} from cellular to continuous sheaf structures. This way, we establish additional interpretation and understanding of the NOFE operator as a sheaf-to-sheaf mapping (for a more detailed elaboration on the connection between NOFE and sheaves, see Appendix \ref{app:sheaf_snn_appendix}).

 We evaluate NOFE across different datasets from climate modeling, medical imaging, and remote sensing, comparing it against relevant methods such as PCA, t-SNE, UMAP, PHATE, SpatialPCA and CONCORD. 
 Our results demonstrate that NOFE outperforms baselines on metrics related to local structure preservation, patch stitching error and consistency, while maintaining competitive global structure preservation. NOFE resolves fine-grained structures and produces smooth, consistent embeddings that generalize, addressing key limitations of discrete reduction methods.

In summary, our key contributions are: 
\begin{enumerate}
	\item We propose a new perspective on dimensionality reduction by formulating it as a sheaf morphism that embeds underlying continuous structure rather than the discrete point samples. 
	\item We introduce NOFE as a NO-based dimensionality reduction method, respecting domain-dependence and continuity.
	\item We provide a theoretical foundation to extend classical SNNs to the continuous domain using GKOs to approximate mappings between continuous sheaves.
    \item We develop new metrics to evaluate consistency of embeddings with structural properties of high-dimensional input space.
	\item We extensively evaluate NOFE and compare with relevant methods, showcasing better properties in terms of continuity, consistency, and preservation of local properties.
\end{enumerate}

\section{Related Work}

Many dimensionality reduction methods construct low-dimensional representations from geometric or statistical relationships between observations in feature space. Methods such as PCA, t-SNE, and UMAP preserve different properties like variance, similarity or neighborhood-structure \cite{2002PrincipalComponentAnalysis, maaten2008VisualizingDataUsing, mcinnes2020UMAPUniformManifold}. When observations are associated with an underlying domain, as in functional data, it is important to distinguish between geometry in feature space and continuity with respect to the domain on which the functions are defined.
Diffusion-based methods like diffusion maps, vector diffusion maps or PHATE \cite{coifman2006DiffusionMaps, singer2012VectorDiffusionMaps, moon2019VisualizingStructureTransitions} typically build graphs based on distances in feature space. Related representation-learning methods such as CONCORD \cite{zhu2026RevealingCoherentCellstate} can likewise use feature-space neighborhoods to guide learning. Thus, their notion of locality is induced by feature values rather than the underlying functional domain.

To incorporate structure beyond isolated data points, a number of approaches introduce spatial or geometric regularity into the embedding process. 
Methods such as GraphPCA \cite{yang2024GraphPCAFastInterpretable}, smooth PCA on manifolds \cite{eardilila2016SmoothPrincipalComponent}, SpatialPCA \cite{shang2022SpatiallyAwareDimension}, graph-regularized autoencoders \cite{liao2017GraphRegularizedAutoEncoders}, and spatial-spectral embeddings \cite{isprs-archives-XLIII-B3-2020-423-2020} incorporate domain relationships through graph constraints, smoothing, covariance structure, or spatial neighborhoods.
However, domain information enters through covariance, regularization, or neighborhood structure over sampled observations rather than by directly parameterizing the reduced representation as a function over the continuous domain.

Functional PCA (fPCA) \cite{ramsay2005FunctionalDataAnalysis} and function-space autoencoders \cite{JMLR:v26:25-0035} treat observations as functions and represent each function through a finite-dimensional latent representation. This yields a global low-dimensional representation of the entire function rather than a domain-indexed representation at individual coordinates.

More recently, NO methods such as DeepONet \cite{lu2021DeepONetLearningNonlinear}, Fourier Neural Operators \cite{li2021FourierNeuralOperator}, and their variants have been developed to learn mappings between infinite-dimensional function spaces that generalize across discretizations, allowing for resolution-independent evaluation and continuous outputs. However, their primary objective is operator learning, particularly for physical systems, solving partial differential equations or super-resolution, rather than dimensionality reduction \cite{oommen2025IntegratingNeuralOperators, calvello2025ContinuumAttentionNeural, wang2024LatentNeuralOperator, leinonen2024ModulatedAdaptiveFourier, jafarzadeh2024PeridynamicNeuralOperators, wei2023SuperResolutionNeuralOperator, kovachki2023NeuralOperatorLearning}. 
NO frameworks such as LNO \cite{wang2024LatentNeuralOperator} or CViT \cite{wang2024CViTContinuousVision} use latent representations internally, while variational autoencoding neural operators \cite{seidman2023VariationalAutoencodingNeural} and nonlinear model reduction approaches for operator learning \cite{eivazi2024NonlinearModelReduction} explicitly compress functional inputs into latent representations, similar to fPCA. But their outputs typically represent a prediction or reconstruction of the target function, not a pointwise low-dimensional representation of the input function. In particular, the dimensionality-reduced representation itself is not expressed as a function over the input domain. 

In parallel, SNNs \cite{hansen2020SheafNeuralNetworks, barbero2022SheafNeuralNetworks, bodnar2022NeuralSheafDiffusion} introduce algebraic structures for modeling local consistency across discrete domains. These approaches encode relationships between local sections via restriction maps and enforce compatibility conditions across overlapping regions. While they provide a principled way to incorporate locality and structure, they are typically defined on discrete base spaces and do not directly address continuous function-valued dimensionality reduction.

In contrast, NOFE views dimensionality reduction as a function-to-function mapping problem in a setting with continuous sheaf structure and uses a NO approach to learn this mapping, enabling continuous, domain-aware dimensionality reduction.

\section{NOFE: Neural Operator Function Embedding for Continuous Domain-aware Dimensionality Reduction}
Assume there is a continuous function $f:\mathcal{M} \rightarrow \mathbb{R}^{d_f}$ over some manifold $\mathcal{M}$ that is accessible through sampled observations $D=\{(x_i, y_i)\}_{i=1}^N$ with $y_i = f(x_i)$.
This could be, e.g., multi-channel satellite images over space or machine sensor data over time.
Traditional dimensionality reduction aims to reduce the dataset by finding a new representative point $z_i \in \mathbb{R}^{d_g}$ for each $y_i \in \mathbb{R}^{d_f}$ with $d_g < d_f$ in a one-to-one correspondence. 
Existing methods view $D$ as a set of individual datapoints in $\mathbb{R}^{d_f}$, completely disregarding properties like continuity or domain relationship, which can lead to well known shortcomings, like discontinuity, overclustering or loss of domain awareness \cite{bergam2026TSNEExaggeratesClusters, liu2025AssessingImprovingReliability}.

\paragraph{Problem Formulation.}
We aim to learn an operator $\mathcal{R}$ that performs dimensionality reduction in function space by mapping $(\mathcal{R} f)(x)=g(x)$ with $g:\mathcal{M} \rightarrow \mathbb{R}^{d_g}$ and $d_g < d_f$ directly reducing the function's codomain, irrespective of sampling. This way, continuity and domain relationships are preserved; embeddings of $y_i$ are given by $z_i = g(x_i)$.

With NOFE, we adopt a GKO \cite{li2020NeuralOperatorGraph} approach to approximate the operator $\mathcal{R}$. This provides great flexibility, as it allows, for example, for irregularly sampled data.
Using a dataset $D=\{(x_i, y_i)\}_{i=1}^N$, where $X=\{x_i\}_{i=1}^N \subset \mathcal{M}$ and $y_i=f(x_i)$, a graph $G=(V,E)$ of nodes $V=\{i\}$ and edges $E$ is constructed with one node $i$ per datapoint.
The node features are given by the function values $y_i$.
Edges are constructed based on domain relationship.
Node $i$ is connected to any node $j \in N(i)$, where $N(i) = \{j \mid x_j \in B_r(x_i)\}$ and $B_r(x)$ denotes the neighborhood around $x$ in $\mathcal{M}$.
Each edge is equipped with attributes $e_{ij}$ containing local information about the nodes' relation. Here we choose to use absolute and relative coordinates as attributes, such that the set of edges $E$ can be written as
\begin{equation}
	E = \{(i,j) : j \in N(i)\}, \quad
	e_{ij} = \{x_i, x_i - x_j\}, \; \forall (i,j) \in E.
\end{equation}
The GKO itself consists of three main components:
\begin{enumerate}
	\item A lifting step $h_i = L y_i$ with $L$ being a single neural layer applied on the node features.
	\item $T$ iterative message passing steps:
		\begin{equation} \label{eq: message passing}
			h_i^{(t+1)}= Wh_i^{(t)} + \sum_{j \in N(i)\cap X} K(i,j)\, h_j^{(t)}, \quad \forall i \in X,
		\end{equation}
		with the kernel matrix $K$, which is predicted by a Multi-Layer Perceptron (MLP) based on the edge attributes $e_{ij}$.
	\item A projection $z_i = P h_i^{(T)}$, with $P$ being another single neural layer, determining the output dimension of the output function $g(x)$.
\end{enumerate}

To find a function $g(x)$ that embeds $f(x)$ in a lower dimensional space in a meaningful way, we take inspiration from t-SNE, which aims to retain neighborhood relation in the two spaces.
Unlike t-SNE, which disregards any domain dependence and compares points solely in the feature space (mapping points that are similar in the input space to similar values in the output space), we take a different approach and compare the relation of node features that are neighbors in the domain $\mathcal{M}$. By doing so, we avoid inconsistent results which are inherent to methods like t-SNE, in particular when sampling points from different sections of a function.

Towards this end, we use a student-t-kernel to calculate pairwise affinities between points in the high- and low-dimensional space
\begin{align} \label{eq: affinity}
p_{ij} = \frac{(1 + \|y_i - y_j\|^2)^{-1}}{\sum\limits_{\substack{(i,k)\in E \\ k\neq j}} (1 + \|y_i - y_k\|^2)^{-1}}, \quad & 
q_{ij} = \frac{(1 + \|z_i - z_j\|^2)^{-1}}{\sum\limits_{\substack{(i,k)\in E \\ k\neq j}} (1 + \|z_i - z_k\|^2)^{-1}} &
\forall (i,j) \in E.
\end{align}
We use the numerical version of the Kullback-Leibler (KL) divergence
\begin{equation} \label{eq: KL divergence}
	D_{\mathrm{KL}}(P \,\|\, Q)
= \sum_{i \ne j} p_{ij}\,\log \frac{p_{ij}}{q_{ij}}
\end{equation}
as a loss function.
This way the operator optimizes the learnable parameters of its neural layers $L$, $W$, $P$ and the MLP predicting the kernel matrix $K(i,j)$ (Equation \ref{eq: message passing}), to find embeddings that retain local statistical properties of the input function.

To give additional intuition and better interpretation, we will introduce a sheaf-theoretic perspective on the described problem setting and our proposed framework.
\begin{definition}[Sheaf \cite{curry2014SheavesCosheavesApplications, bredon1967SheafTheory}]
A sheaf $\mathcal{F}$ over a topological space $\mathcal{M}$ assigns to each open set $U \subseteq \mathcal{M}$ a set of objects $\mathcal{F}(U)$, called sections over $U$, together with:

\textbf{(i) Restriction:} For every inclusion $V \subseteq U$, there exists a map
\[
\rho_{U,V}: \mathcal{F}(U) \rightarrow \mathcal{F}(V),
\]
which restricts sections from larger to smaller domains.

\textbf{(ii) Gluing:} If a collection of local sections $\{s_i \in \mathcal{F}(U_i)\}_{i \in I}$ defined on an open cover $\{U_i\}_{i \in I}$ agrees on overlaps, i.e.
\[
\rho_{U_i, U_i \cap U_j}(s_i) = \rho_{U_j, U_i \cap U_j}(s_j),
\]
then there exists a unique global section $s \in \mathcal{F}(\bigcup_i U_i)$ such that $s|_{U_i} = s_i$ for all $i$.
\end{definition}

\begin{definition}[Sheaf morphism \cite{2022CellularSheafCohomology}] \label{def: sheaf morphism}
	A sheaf morphism $\Phi : \mathcal{F}\rightarrow \mathcal{G}$ is a collection of maps
	\begin{equation}
		\Phi _U : \mathcal{F}(U)\rightarrow \mathcal{G}(U) \quad \forall \ U \subseteq \mathcal{M},
	\end{equation}
	that commute with restrictions
	\begin{equation}
		\Phi _V \circ \rho_{U,V}^{\mathcal{F}} = \rho_{U,V}^{\mathcal{G}} \circ \Phi _U \quad \forall \ V \subseteq U.
	\end{equation}
	Here, $\rho_{U,V}^{\mathcal{F}}$ and $\rho_{U,V}^{\mathcal{G}}$ are the restriction maps from $U$ to $V$ for the sheaves $\mathcal{F}$ and $\mathcal{G}$, respectively.
\end{definition} 

\begin{definition}[Locality] \label{def: locality}
	An operator $\Phi$ is called local if for every open set $U \subseteq \mathcal{M}$ and all
	$s,t \in \Gamma (\mathcal{M}, \mathcal{F})$ satisfying $s|_U = t|_U$, it holds that
	\begin{equation}
		\Phi(s)|_U = \Phi(t)|_U.	
	\end{equation}
\end{definition}

\begin{definition}[Sheaf of continuous functions]
	For any open set $U \subseteq \mathcal{M}$, let $\mathcal{C}_d (U) := C(U, \mathbb{R}^{d} ) = \{f:U\rightarrow \mathbb{R}^{d_f} :f$ is continuous$\}$.
	The assignment $U \rightarrow \mathcal{C}_d(U)$ defines the sheaf of continuous, $\mathbb{R}^{d}$-valued functions on $\mathcal{M}$. 
\end{definition}

\begin{definition}[Global sections \cite{bredon1967SheafTheory}]
	For a given sheaf $\mathcal{F}(\mathcal{M})$ we define $\Gamma (U, \mathcal{F})$ as the set of sections $\mathcal{F}(U)$ over any open set $U \subseteq \mathcal{M}$.
	The set $\Gamma (\mathcal{M}, \mathcal{F})$ is called the set of global sections.
\end{definition}

\begin{proposition}
	For any continuous function $f:\mathcal{M} \rightarrow \mathbb{R}^{d}$ there is a sheaf $\mathcal{F}_{d}(\mathcal{M}) \subseteq \mathcal{C}_{d} (\mathcal{M})$, such that $f \in \mathcal{F}_{d}(\mathcal{M})$ for all possible $f$.
\end{proposition}
\begin{proof}
	This follows trivially from the restrictions on $f$ to be at least a continuous function. Additional restrictions may apply depending on the data. 
\end{proof}

\begin{assumption} \label{ass: global sections}
	Any function $f$, giving rise to our dataset $D$, is a random global section of $\mathcal{F}_{d_f}(\mathcal{M}) \subseteq \mathcal{C}_{d_f} (\mathcal{M})$, drawn from a distribution that can be written as
	\begin{equation}
		f \sim \mathcal{P} \left(\Gamma \left(\mathcal{M}, \mathcal{F}_{d_f} \right) \right).
	\end{equation}
	Structure of the sheaf, as well as the distribution $\mathcal{P}$ are given by the nature of the data.
\end{assumption}
This assumption is motivated by the idea that we are not dealing with mathematical constructions but real world data that follows the laws of physics.

\paragraph{Sheaf-theoretic Problem Formulation.}
Dimensionality reduction of continuous functions, can be seen as a function to function mapping $(\mathcal{R} f)(x)=g(x)$ between sets of global sheaf sections, with the function $f \in \left(\Gamma \left(\mathcal{M}, \mathcal{F}_{d_f} \right) \right)$ and $g \in \left(\Gamma \left(\mathcal{M}, \mathcal{F}_{d_g} \right) \right)$, where $\mathcal{F}_{d_f} \subseteq \mathcal{C}_{d_f}(\mathcal{M})$ and $\mathcal{F}_{d_g} \subseteq \mathcal{C}_{d_g}(\mathcal{M})$.
To ensure consistency across maps on any subsets $V \subseteq U \subseteq \mathcal{M}$, $\mathcal{R}$ is also required to satisfy the commutation condition 
\begin{equation}
\mathcal{R}_V \circ \rho_{U,V}^{\mathcal{F}_{d_f}} = \rho_{U,V}^{\mathcal{F}_{d_g}} \circ \mathcal{R} _U 
\end{equation}
and thereby defining a sheaf morphism.

\noindent
SNNs are by definition, data-driven maps that work on sheaf structures \cite{hansen2020SheafNeuralNetworks} and therefore seem a natural fit to approach this problem.
But classic SNNs work on cellular sheaves and therefore assume a discrete data structure, while our goal is to explicitly account for continuity.
However, we observe a close connection between GKOs and SNNs from an architectural perspective:
In an SNN, information is transported between neighboring stalks via multiplication with matrix-valued restriction maps. In practice this is done with a message-passing type process between neighboring nodes on a graph. In the more general framing of Neural Sheaf Diffusion  \cite{bodnar2022NeuralSheafDiffusion} the restriction maps are predicted via a separate neural network based on local graph features. The resulting architecture is almost identical to the GKO. While \citet{bodnar2022NeuralSheafDiffusion} use the features of neighboring nodes as basis for predicting the restriction maps, the GKO uses the attributes of the corresponding edge. The construction of node features and edge attributes gives enough freedom of choice to regard the architectures as equivalent. However, Neural Sheaf Diffusion -- like its related SNN frameworks -- effectively works on fixed, discrete data, while the GKO assumes a continuous data structure and allows for predictions at arbitrary query locations. Therefore, we propose the use of GKOs as a generalization of SNNs to model mappings between continuous sheaves. This approach may be interpreted as a Sheaf Neural Operator. We formalize this in the following. 
A more detailed connection between the frameworks is provided in Appendix \ref{app:sheaf_snn_appendix}.

\begin{proposition}[Local operator] \label{prop: locality}
	By using a GKO as described above, NOFE is modeling a local operator in the sense of Definition \ref{def: locality}.
\end{proposition}
\begin{proof}
    For any subset $U \subseteq \mathcal{M}$, $R_U$ and the associated kernel \(K\) has support only on indices corresponding to an arbitrarily small neighborhood \(B_r\).
\end{proof}

\begin{proposition}
	In our setting, NOFE --- while not being a sheaf morphism in strict mathematical sense --- is compatible with sheaf structures and may be interpreted as an approximation of a stalk-dimension reducing operator, acting on continuous sheaf structures.
\end{proposition}
Following this interpretation, we expect NOFE embeddings to satisfy criteria of sheaf structure, like restriction and gluing. We demonstrate NOFE's numerical compatibility with sheaf structure in section \ref{sec: gluing}.

\section{Experiments} \label{sec: experiments}

We evaluate NOFE on multiple datasets, underlining the generality of the method but will in this section highlight results for the ERA5\footnote{\url{https://cds.climate.copernicus.eu/datasets/reanalysis-era5-pressure-levels?tab=overview}} climate reanalysis dataset, as it provides the most intuitive separation between domain and feature space with a clear interpretation of features and results.
Additional results including diffusion MRI and satellite images are provided in Appendix \ref{app: datasets}.
Our ERA5 dataset consists of 11 climate variables at a pressure level of 550 hPa over a spatial grid covering Europe (180 $\times$ 180 resolution).
We generate samples by randomly drawing 1000 -- 10000 points from the grid.
A training dataset is constructed with 10 samples per day for each day of 2018.
All evaluations are performed by generating one sample per day from January 2019.

For comparison, we project data into 3-dimensional space, enabling visualization by interpreting the three dimensions as RGB channels and interpolating predictions onto a regular grid.
As the selection of color channels for any embedding dimension is arbitrary, these colors cannot be interpreted directly and merely serve the visualization of resulting data structures.  
While we are not aware of any other methods that treat dimensionality reduction as a domain-aware and continuous function-to-function mapping, we choose PCA, t-SNE and UMAP as standard baselines. Additionally, we compare to PHATE as another graph-based approach, CONCORD, and SpatialPCA, which also incorporates domain-relationships into the embeddings.

We first provide experiments with a point-to-point correspondence between input and output data for a direct comparison across methods, using metrics for global and local structure of the embeddings.
Subsequently, we perform a super-resolution experiment showcasing NOFE's capability to predict embeddings at arbitrary query locations beyond mere interpolation.
Implementation details, including parameter choices and algorithms are deferred to Appendix \ref{app: implementation}.
The code base is available online\footnote{\url{https://github.com/laueb4818/NOFE}}.

\begin{table}[t]
  \centering
  \caption{Quantitative comparison on the ERA5 climate reanalysis dataset. We assess preservation of global and local feature geometry, consistency across independently embedded spatial patches, super-resolution embeddings, and gluing errors. Values are mean $\pm$ standard deviation; lower is better for error and stress metrics, whereas higher is better for global correlation. ``Aligned'' denotes post-processing with Procrustes alignment. Bold and underlined entries are the best and second-best finite results, respectively, within each metric.}
  \label{tab:era5-quantitative-comparison-transposed}
  \resizebox{\textwidth}{!}{%
  \begin{tabular}{llccccccc}
    \toprule
    Setting & Metric & NOFE & PCA & t-SNE & UMAP & PHATE & SpatialPCA & CONCORD \\
    \midrule
    \multirow{2}{*}{\makecell[l]{Embedding\\quality}}
      & Stress-1$\downarrow$ & \underline{0.379} $\pm$ 0.015 & \textbf{0.268} $\pm$ 0.014 & 4.176 $\pm$ 0.812 & 1.079 $\pm$ 0.222 & 0.994 $\pm$ 0.001 & 1.006 $\pm$ 0.863 & 0.840 $\pm$ 0.017 \\
      & Stress-local$\downarrow$ & \textbf{0.125} $\pm$ 0.012 & \underline{0.464} $\pm$ 0.054 & 2.127 $\pm$ 0.479 & 0.694 $\pm$ 0.044 & 0.998 $\pm$ 0.001 & 1.337 $\pm$ 0.777 & 0.931 $\pm$ 0.021 \\
    \midrule
  \multirow{2}{*}{\makecell[l]{Patch\\stitching}}
      & $SE$-local$\downarrow$ & \underline{0.823} $\pm$ 0.047 & 0.843 $\pm$ 0.110 & 0.892 $\pm$ 0.097 & 1.126 $\pm$ 0.199 & \textbf{0.821} $\pm$ 0.237 & 1.403 $\pm$ 0.127 & 1.509 $\pm$ 0.411 \\
      & $SE$-region$\downarrow$ & \textbf{13.035} $\pm$ 0.852 & 22.180 $\pm$ 4.887 & 59.030 $\pm$ 14.271 & 267.586 $\pm$ 48.483 & 106.394 $\pm$ 39.014 & \underline{20.442} $\pm$ 2.251 & $\infty$ \\
    \midrule
    \multirow{2}{*}{Gluing}
      & MSE (unaligned)$\downarrow$ & \textbf{0.080} $\pm$ 0.047 & 1.369 $\pm$ 0.701 & 1.642 $\pm$ 0.586 & 1.827 $\pm$ 0.884 & 1.704 $\pm$ 0.705 & \underline{1.346} $\pm$ 0.718 & 1.432 $\pm$ 0.720 \\
      & MSE (aligned)$\downarrow$ & --- & 0.339 $\pm$ 0.328 & 0.403 $\pm$ 0.153 & \textbf{0.245} $\pm$ 0.119 & \underline{0.277} $\pm$ 0.170 & 0.472 $\pm$ 0.324 & 0.363 $\pm$ 0.223 \\
    \midrule
    \multirow{2}{*}{\makecell[l]{Super\\-resolution}}
          & Stress-1$\downarrow$ & 0.485 $\pm$ 0.031 & \underline{0.394} $\pm$ 0.024 & 3.864 $\pm$ 4.421 & 0.452 $\pm$ 0.028 & 0.982 $\pm$ 0.001 & \textbf{0.391} $\pm$ 0.035 & 0.889 $\pm$ 0.044 \\
      & Stress-local$\downarrow$ & \textbf{0.666} $\pm$ 0.024 & 0.763 $\pm$ 0.017 & 2.092 $\pm$ 1.888 & 0.826 $\pm$ 0.018 & 0.996 $\pm$ 0.001 & \underline{0.749} $\pm$ 0.030 & 0.957 $\pm$ 0.021 \\
    \bottomrule
  	\end{tabular}
  	}%
\end{table}

\subsection{Embedding Structure and Relationships to Input Variables}
We examine how well data structures of the high-dimensional space are retained in the low-dimensional space.
We use the Kruskal Stress-1 (Eq. \ref{eq: stress}), with $d_{ij}^a = \lVert a_i - a_j \rVert$ as a common metric to estimate the global preservation of feature distances in the low dimensional space \citep{borg2005MDSModelsMeasures, ayesha2020OverviewComparativeStudy}.
Further, we define Stress-local (Eq. \ref{eq: local-stress}), as a similar metric that only considers relations of domain-neighbors.
The latter metric aligns closer with the aim of our methods, as NOFE is set up to retain local statistics.
While results in Table \ref{tab:era5-quantitative-comparison-transposed} show that PCA preserves feature distances best globally, NOFE remains competitive on Stress-1 and yields the lowest stress values on local neighbors.

\begin{subequations}
\begin{minipage}{0.45\textwidth}
\begin{equation}
\mathrm{Stress\text{-}1}=\sqrt{\frac{\sum_{i<j}\left(d_{ij}^y-d_{ij}^z\right)^2}{\sum_{i<j} (d_{ij}^y)^2}}
\label{eq: stress}
\end{equation}
\end{minipage}
\hfill
\begin{minipage}{0.49\textwidth}
\begin{equation}
\mathrm{Stress\text{-}local}=\sqrt{\frac{\sum_{(i,j)\in E}\left(d_{ij}^y-d_{ij}^z\right)^2}{\sum_{(i,j)\in E} (d_{ij}^y)^2}}
\label{eq: local-stress}
\end{equation}
\end{minipage}
\end{subequations}

\begin{figure}[h]
  \center
  \includegraphics[width=0.9\textwidth]{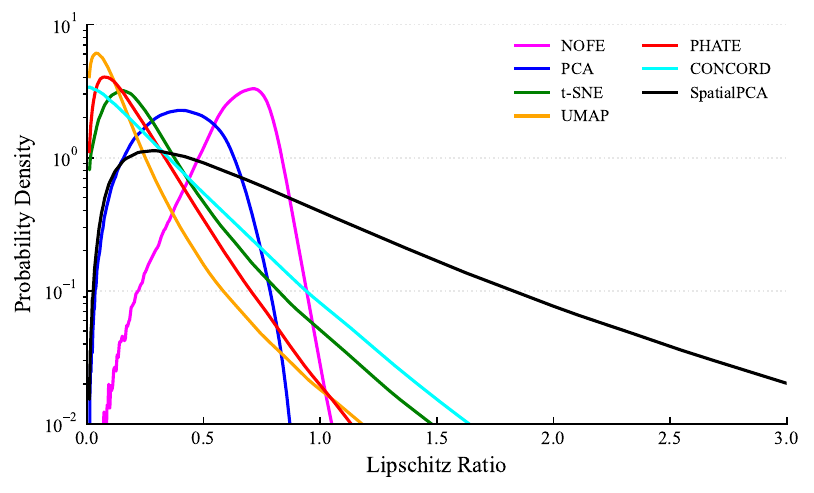}
  \caption{Kernel density estimates of Lipschitz ratios across methods plotted on semi-logarithmic scale.}
  \label{fig: lipschitz-ratio}
\end{figure}

Related to the preservation of local feature distances, we evaluate the preservation of continuity in the low-dimensional space.
Key for a meaningful embedding is not simply being as continuous as possible, as this would be achieved by mapping all points to the same feature value, but rather to reflect the continuity of the input data accurately.
To this end, we take inspiration from the Lipschitz continuity criterion \citep{folland1999RealAnalysisModerna} to define a pairwise Lipschitz continuity measure (Eq. \ref{eq: pw-continuity}) and the Lipschitz ratio (Eq. \ref{eq: Lipschitz-ratio}).
The Lipschitz ratio $r_L(i,j)$ compares how continuous the transition between points $i$ and $j$ is in the low-dimensional space relative to the high-dimensional space.
Since the spaces have different dimensions and the scaling is not normalized, the absolute values of this ratio are not comparable, but the statistics of these values are.
If the continuity of the high-dimensional data is reflected faithfully in the low-dimensional data, the Lipschitz ratios for point pairs $(i,j)$ will be distributed narrowly around a common value.
Figure \ref{fig: lipschitz-ratio} shows kernel-based estimates for the probability densities of Lipschitz ratios $r_L(i,j)$ for all $(i,j) \in E$ accumulated over predictions for every day of January 2019 on a semi-logarithmic scale.
The Figure shows that distributions for most methods peak close to zero, with PCA (\~0.4) and NOFE (\~0.7) having their peaks at slightly larger Lipschitz ratios. 
Further, the Figure indicates heavy-tailed distributions for all methods with the exception of PCA and NOFE, as all other methods show an exponential decay.
PCA yields the most narrow density function with fastest decay, followed by NOFE. 
Both have a similar width of ca. 0.8 and 1.0, respectively.
We draw the conclusion that PCA as well as NOFE perform well on this metric and conserve function continuity the best.

\begin{subequations}
\begin{minipage}{0.45\textwidth}
\begin{equation}
L_{ij}^f = \frac{|f(x_i) - f(x_j)|}{|x_i - x_j|}
\label{eq: pw-continuity}
\end{equation}
\end{minipage}
\hfill
\begin{minipage}{0.49\textwidth}
\begin{equation}
r_L(i,j) = \frac{L_{ij}^g}{L_{ij}^f} = \frac{|z_i - z_j|}{|y_i - y_j|}.
\label{eq: Lipschitz-ratio}
\end{equation}
\end{minipage}
\end{subequations}

To complement the distance and continuity metrics, we analyze the embedding structure by calculating correlations of embedded features to the individual input variables and provide plots of the feature distribution with further interpretations in Appendix \ref{app: feature correlation}.

\subsection{Consistency and Sampling Independence} \label{sec: Consistency and Sampling Independence}
\begin{wrapfigure}{r}{0.34\textwidth}
	\centering
	\includegraphics[width=1.05\linewidth,
	 trim=0 0 0 35
	 ]{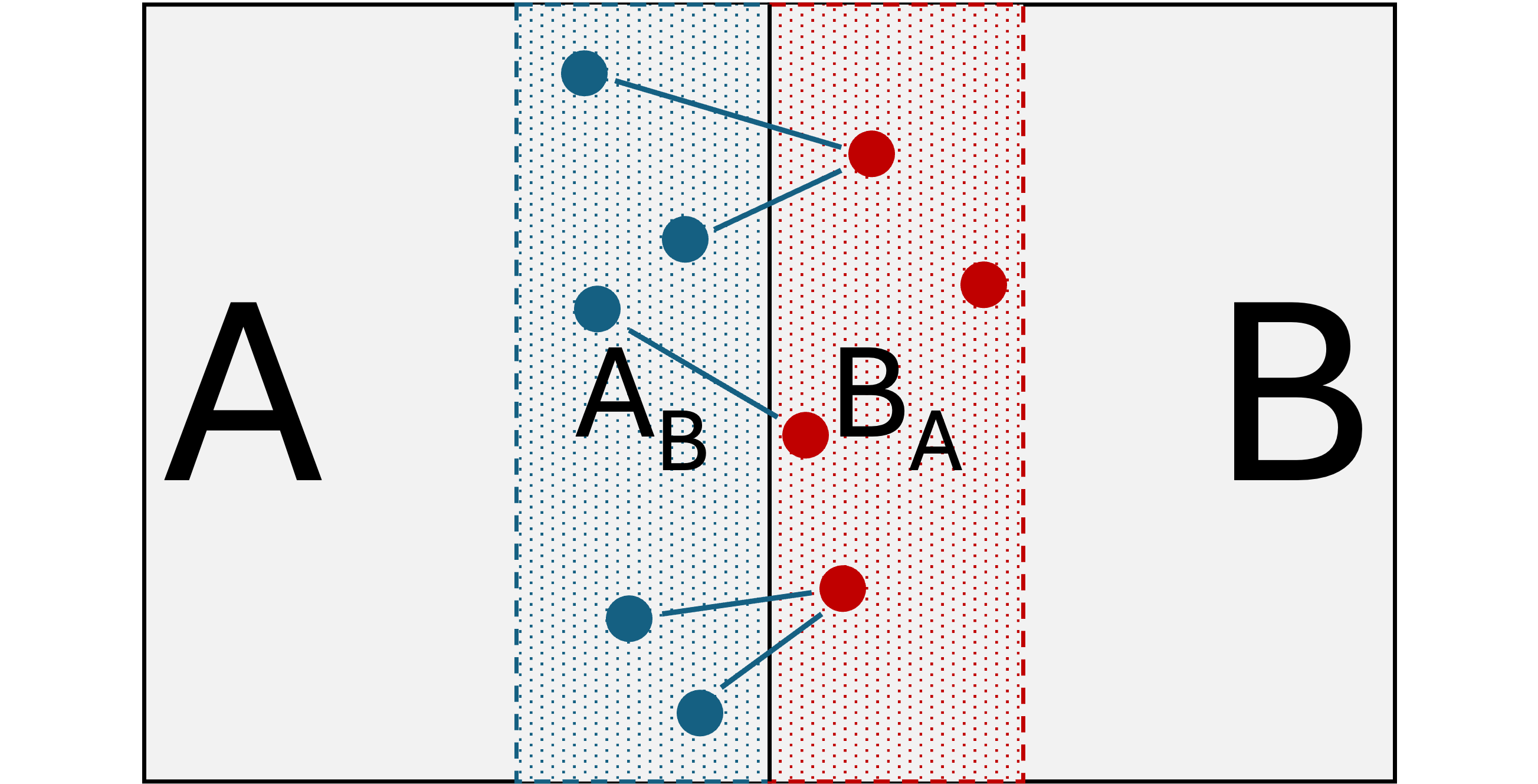}
	\caption{Scheme for regional patches $A$, $B$ and the border regions $A_B$ and $B_A$ between them. Blue lines indicate nearest neighbors from $A_B$ in $B_A$.  }
	\label{fig: patching-scheme}	
	\vspace{-10pt}
\end{wrapfigure}
NOFE works on the level of the continuous structures underlying the data rather than a detached point cloud, making it a sampling independent method.
To demonstrate this important property, we split up the domain in four subregions, sample and reduce random subsets for each region independently and merge the results together subsequently.
An exemplary result can be seen in Figure \ref{fig: patch-stitching results}. Magnified plots of different additional results can be found in Appendix \ref{app: additional visualisations}.
The traditional methods immediately reveal the four subregions, as the embedding spaces --- here seen as color schemes --- are inconsistent, leading to sharp discontinuities between the patches. 
For the NOFE embeddings on the other hand, very close inspection is necessary to see any signs of the separate embedding, as the transition from one patch to another is very smooth and color schemes appear consistent\footnote{We observe the same effect by resampling random subsets from the same date and comparing the results, yielding NOFE embeddings that are consistent, unlike the other methods.}.

To quantify this observation, we define a border region of 2$^\circ$ between every neighboring patches $A$ and $B$, with $A_B$ denoting the border region within $A$ towards patch $B$ and $B_A$ the border region within $B$, respectively (see Fig. \ref{fig: patching-scheme}). 
For each point $a \in A_B$, we find the spatially closest point $b_{N(a)} \in B_A$ and calculate the Euclidean distance $d(z_a, z_{b_{N(a)}})$ between the feature embeddings $z_a$, $z_{b_{N(a)}}$ and vice versa.
For the embeddings to be consistent, we would expect this value to be small, compared to general variations in the data, since the embeddings represent a smooth function.
To account for differences in embedding scale and properties, the feature distance must be normalized.
We consider two choices:
\begin{enumerate}
	\item Region normalization: Each distance $d(z_a, z_{b_{N(a)}})$ is normalized by $\langle d(z_a, z_{a'}) \rangle_{a' \in A_B}$, the mean feature distance of all points inside the border region of the source point.
	\item Neighbor normalization: Each distance $d(z_a, z_{b_{N(a)}})$ across a border, is normalized by a corresponding distance within a border $d(z_a, z_{a_{N(a)}})$, where $z_{a_{N(a)}}$ denotes the feature vector of the closest point to $a$ in $A_B$.
\end{enumerate}

We refer to these metrics as \textit{Patch-Stitching Errors} denoted as $SE$-region and $SE$-local, respectively. The results in Table \ref{tab:era5-quantitative-comparison-transposed} show very close results for $SE$-local amongst NOFE, PCA, t-SNE and PHATE, where PHATE ranks first with 0.821, followed by NOFE with 0.823. UMAP, SpatialPCA and CONCORD have all scores $>$ 1.
The scores for $SE$-region differ more drastically, with NOFE ranking first (13.035) followed by SpatialPCA (20.442) which already exceeds NOFE's score by 57~\%.
\begin{figure}[h]
    \centering
    \begin{subfigure}{0.13\textwidth}
        \centering
        \includegraphics[width=\linewidth]{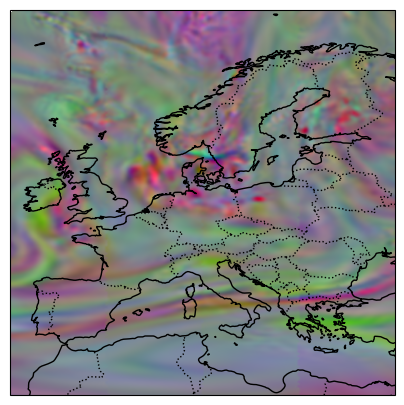}
        \caption{\tiny{NOFE}}
        \label{fig:stitching-a}
    \end{subfigure}
    \hfill
    \begin{subfigure}{0.13\textwidth}
        \centering
        \includegraphics[width=\linewidth]{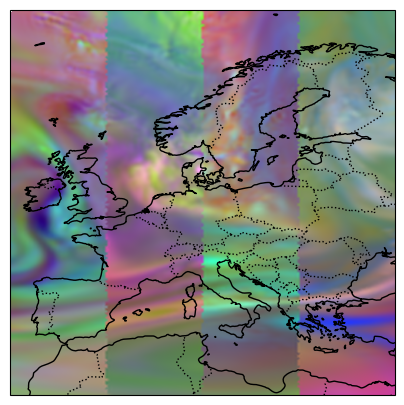}
        \caption{\tiny{PCA}}
        \label{fig:stitching-b}
    \end{subfigure}
    \hfill
    \begin{subfigure}{0.13\textwidth}
        \centering
        \includegraphics[width=\linewidth]{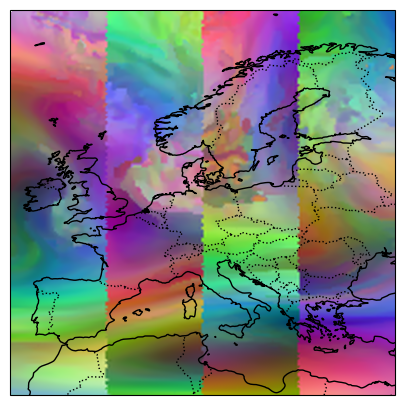}
        \caption{\tiny{t-SNE}}
        \label{fig:stitching-c}
    \end{subfigure}
    \hfill
    \begin{subfigure}{0.13\textwidth}
        \centering
        \includegraphics[width=\linewidth]{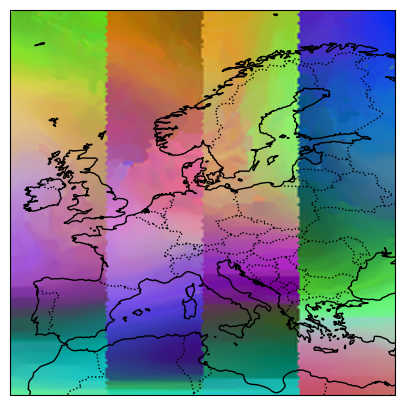}
        \caption{\tiny{UMAP}}
        \label{fig:stitching-d}
    \end{subfigure}
    \hfill
    \begin{subfigure}{0.13\textwidth}
        \centering
        \includegraphics[width=\linewidth]{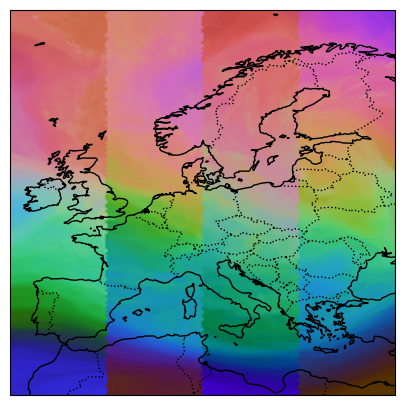}
        \caption{\tiny{PHATE}}
        \label{fig:stitching-e}
    \end{subfigure}
    \hfill
    \begin{subfigure}{0.13\textwidth}
        \centering
        \includegraphics[width=\linewidth]{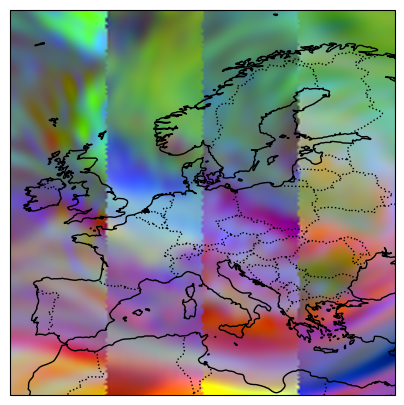}
        \caption{\tiny{SpatialPCA}}
        \label{fig:stitching-f}
    \end{subfigure}
    \hfill
    \begin{subfigure}{0.13\textwidth}
        \centering
        \includegraphics[width=\linewidth]{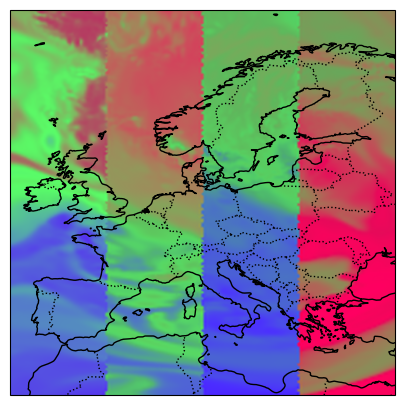}
        \caption{\tiny{CONCORD}}
        \label{fig:stitching-g}
    \end{subfigure}

    \caption{Exemplary visualization of patch stitching results for data sampled from 2019-06-15.}
    \label{fig: patch-stitching results}
\end{figure}

\subsection{Restriction and Gluing Properties} \label{sec: gluing}
Following our theoretical setup, we interpret NOFE as an operator approximating a sheaf morphism between a sheaf of high-dimensional functions to a sheaf of low-dimensional functions.
From this interpretation follows, based on the properties of sheaves (\textit{restriction} and \textit{gluing}) and sheaf morphisms, that for any two subsets $U_k, U_l \subset \mathcal{M}$ with an intersection  $U_{i}=U_k \cap U_l$, the embeddings $z_i \forall i \in U_{i}$ must agree on the overlap, i.e. they must satisfy $R_k y_i = R_l y_i  \forall i \in U_{i}$.
We test this hypothesis, by defining two overlapping regions $A$ and $B$.
For both regions we sample data points of the same date and embed them separately.
Merging the two datasets together and plotting them on one map gives us a qualitative estimate of whether the overlaps agree, or not.
Exemplary results are shown in Figure \ref{fig: gluing}. For NOFE, the combined image of the two regions is indistinguishable from one that has been created from a single dataset, which strongly supports that the sheaf interpretation holds.
For PCA, t-SNE and UMAP on the other hand, not only is the region of overlap clearly visible, but also it is clearly distinguishable which individual points belong to the same section, proving that these approaches do not satisfy the criteria of the sheaf interpretation. Additional, magnified visualizations for all methods can be found in Appendix \ref{app: additional visualisations}.

\begin{figure}[h]
    \centering
    \begin{subfigure}{0.13\textwidth}
        \centering
        \includegraphics[width=\linewidth]{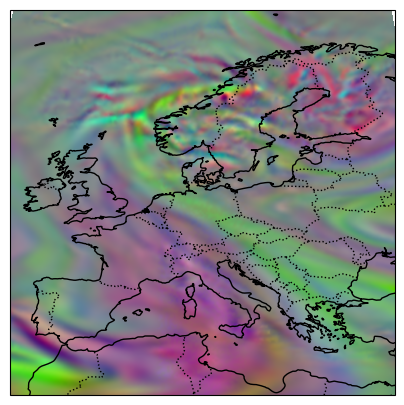}
        \caption{\tiny{NOFE}}
        \label{fig:gluing-a}
    \end{subfigure}
    \hfill
    \begin{subfigure}{0.13\textwidth}
        \centering
        \includegraphics[width=\linewidth]{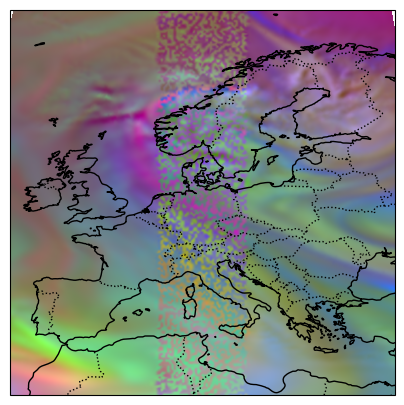}
        \caption{\tiny{PCA}}
        \label{fig:gluing-b}
    \end{subfigure}
    \hfill
    \begin{subfigure}{0.13\textwidth}
        \centering
        \includegraphics[width=\linewidth]{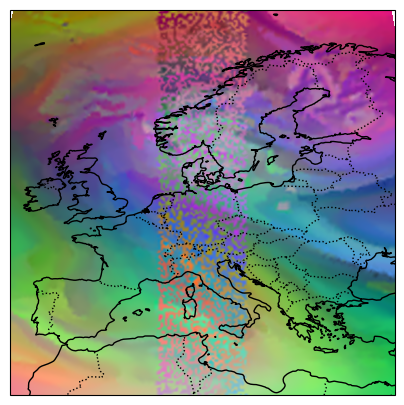}
        \caption{\tiny{t-SNE}}
        \label{fig:gluing-c}
    \end{subfigure}
    \hfill
    \begin{subfigure}{0.13\textwidth}
        \centering
        \includegraphics[width=\linewidth]{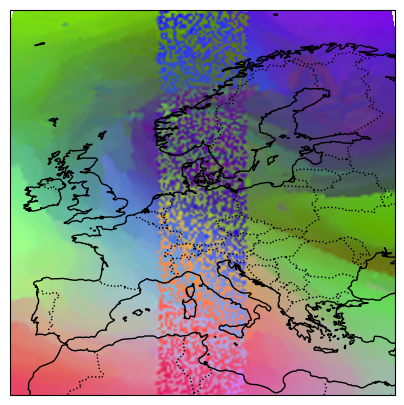}
        \caption{\tiny{UMAP}}
        \label{fig:gluing-d}
    \end{subfigure}
    \hfill    
    \begin{subfigure}{0.13\textwidth}
        \centering
        \includegraphics[width=\linewidth]{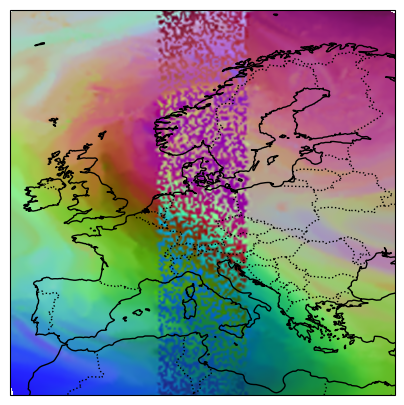}
        \caption{\tiny{PHATE}}
        \label{fig:gluing-e}
    \end{subfigure}
    \hfill    
    \begin{subfigure}{0.13\textwidth}
        \centering
        \includegraphics[width=\linewidth]{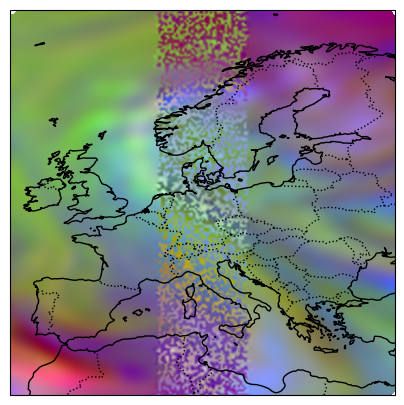}
        \caption{\tiny{SpatialPCA}}
        \label{fig:gluing-f}
    \end{subfigure}
    \hfill    
    \begin{subfigure}{0.13\textwidth}
        \centering
        \includegraphics[width=\linewidth]{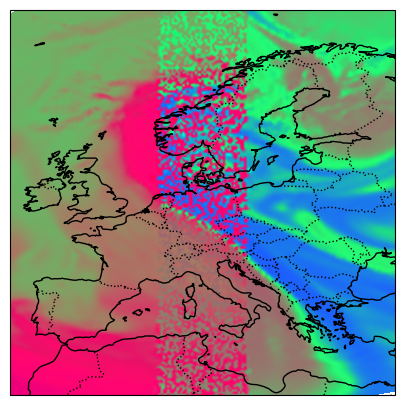}
        \caption{\tiny{CONCORD}}
        \label{fig:gluing-g}
    \end{subfigure}

    \caption{Qualitative experiment to test the gluing properties of low-dimensional embeddings after separately reducing two overlapping sections of data points.}
    \label{fig: gluing}
\end{figure}

For a quantitative evaluation of the gluing property, we run this experiment again, using all available data points, such that every point within the overlap region is embedded twice, i.e. once per patch.
To evaluate the agreement of the overlapping embeddings, we apply mean squared error $\mathrm{MSE}(R_k y_i, R_l y_i)  \forall i \in U_{i}$.
In order to make results between methods comparable, the embeddings are normalized to mean 0 and standard deviation 1 beforehand.
Additionally, we run the same experiment but apply Procrustes alignment \cite{vosdewael2020BrainSpaceToolboxAnalysis} on the patches in order to compensate for differences in orientation and scale of the embedding dimensions. This step is omitted for NOFE as our method projects independently embedded data into the same lower-dimensional space by default.
Averaged over samples from January 2019, gluing MSEs are shown in Table \ref{tab:era5-quantitative-comparison-transposed}. NOFE's MSE (0.080) is significantly lower than for remaining methods. While the Procrustes alignment improves the scores by almost one order of magnitude, the best performing method with alignment (UMAP with a score of 0.245) does not come close to NOFE's score without additional alignment step.

\subsection{Super-Resolution} \label{sec: superresolution}
After comparison to standard methods in the established point-to-point reduction setting, we will now demonstrate NOFE's unique capability of predicting reduced function values $z_q$ at arbitrary query locations $x_q$, even if the high-dimensional data for that location is not provided.
We sample a set of $N_i$ input points and define a set $X_q$ of $N_q$ output locations, with $N_q > N_i$.
For visualization we choose $N_i=400, N_q=5000$ and for quantitive experiments $N_i=200, N_q=1000$.
While NOFE can predict embeddings $z_q$ at any location directly, existing methods rely on embeddings $z_i$ from location given in the data from which query embeddings $z_q$ can only be interpolated. 
In this experiment, we choose to do that via Delaunay-based linear interpolation (LinearNDInterpolator, SciPy) after applying the respective dimensionality reduction method, as we consider this a realistic real-world approach. 
In Appendix \ref{app: Superresolution}, we consider other interpolation schemes on the high-dimensional data in order to make the approach more comparable to the way NOFE initializes new points for locations of unseen data.
Figure \ref{fig: superresolution} shows NOFE embeddings for high resolution query locations $X_q$ based on low resolution input locations $X_i$.
For comparison, we show results of PCA reduction applied on both low and high resolution data samples.
The low resolution PCA results appear blurry with clear artifacts from interpolation, replaced with sharper and more detailed structures when PCA is provided with the actual high resolution query locations.
NOFE on the other hand, while only using the low resolution input, does not show any artifacts and, while being slightly blurred and not resolving fine details, the prediction seems much more natural and resembles structures seen in Figure \ref{fig: superresolution-third} closer. 
More visualizations of super-resolution embeddings are provided in Appendix \ref{app: Superresolution}.
\begin{figure}[h]
    \centering
    \begin{subfigure}{0.25\textwidth}
        \centering
        \includegraphics[width=\linewidth]{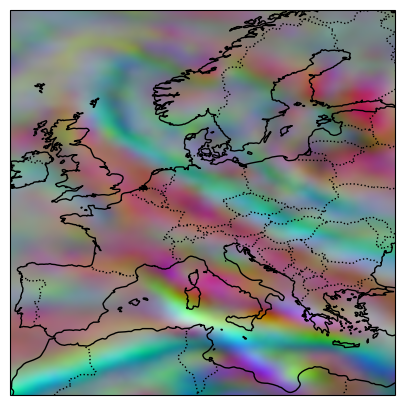}
        \caption{NOFE}
        \label{fig: superresolution-first}
    \end{subfigure}
    \hfill
    \begin{subfigure}{0.25\textwidth}
        \centering
        \includegraphics[width=\linewidth]{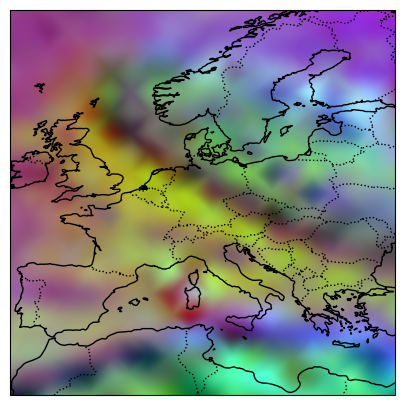}
        \caption{PCA ($N_i=400$)}
        \label{fig: superresolution-second}
    \end{subfigure}
    \hfill
    \begin{subfigure}{0.25\textwidth}
        \centering
        \includegraphics[width=\linewidth]{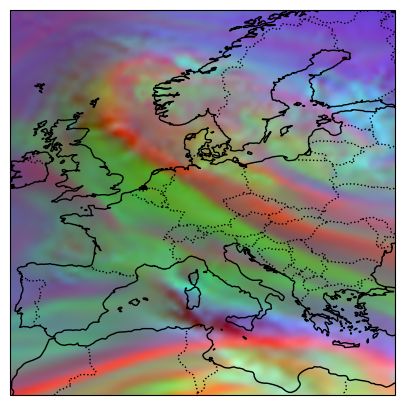}
        \caption{PCA ($N_i=5000$)}
        \label{fig: superresolution-third}
    \end{subfigure}
    \caption{NOFE applied in super-resolution setting, mapping data from a set of $N_i = 400$ input locations $X_i$ to $N_q = 5000$ query locations $X_q$ (Figure \ref{fig: superresolution-first}). For comparison we see PCA results based on the same input locations (Figure \ref{fig: superresolution-second}) and results for PCA using the values from the 5000 query locations $X_q$ as input (Figure \ref{fig: superresolution-third}).}
    \label{fig: superresolution}
\end{figure}
For a numerical evaluation, we calculate local and global stress values for all seven methods based on predictions for query locations on samples from January 2019, seen in Table \ref{tab:era5-quantitative-comparison-transposed}.
Results align largely with previous observations, showing NOFE outperforming t-SNE, PHATE and CONCORD on both, and PCA, UMAP and SpatialPCA on the local stress metric.

\section{Conclusion}
In this work we introduced NOFE, a domain-aware framework that reformulates dimensionality reduction as a sheaf morphism, moving beyond point-cloud embeddings to learn function-to-function mappings that respect continuous domain structure. Through graph kernel operators, NOFE achieves mesh-free, resolution-independent evaluation at arbitrary query locations, bridging neural operator learning with sheaf-theoretic representations. Empirically, NOFE significantly outperforms classical baselines in preserving local geometry and sampling independence, while confirming its theoretical role as an approximate stalk-dimension-reducing operator on continuous sheaves.
\textbf{Limitations:} Despite its novel formulation, our method’s scaling depends on spatial points and neighborhood size, leading to increased memory and training time in high-resolution domains (e.g., $180\times180$ grids), particularly when using the extended subgraph approach. Furthermore, performance is sensitive to hyperparameters (\textit{k}, message passing iterations \textit{T}), necessitating dataset-specific tuning. Theoretically, the method relies on a meaningful distance metric for graph construction; its effectiveness diminishes in highly non-Euclidean manifolds where intrinsic geometry must be approximated via geodesic distances. 
Looking forward, we aim to extend NOFE to more complex topological domains, integrate adaptive mesh refinement, and explore its potential in other tasks and settings.

\begin{ack}
This work was partially funded by the Research Council of Norway (Visual Intelligence, grant no. 309439 and KnowEarth, grant no. 337481). The work was also supported by the Pioneer Centre
for AI under DNRF grant no. P1.
\end{ack}

\bibliographystyle{plainnat}   
\bibliography{NOFE_final}      

\appendix
\section{Implementation Details} \label{app: implementation}

NOFE uses a GKO approach, which requires data in a graph structure. The graph is constructed based on the domain structure of sample points $x$. For a simple point-to-point correspondence $X_i = X_q$ between locations $X_i$ of input samples and query locations $X_q$, a graph is constructed using the set $X_i$. Each sample point is represented by a node $i$, function values are represented by node features. Graph edges are defined based on domain distance, with neighbors of node $i$ defined as $N(i) = \{j \mid x_j \in B_r(x_i)\}$ where $B_r(x)$ denotes a hypersphere of radius $r$ around $x$ in $\mathcal{M}$. However, for practical reasons we choose to implement neighborhood based on a $k$-nearest-neighbor search \citep{t.cover1967NearestNeighborPattern}. This breaks locality as per Definition \ref{def: locality}, which was stated in Proposition \ref{prop: locality} and only holds in a graph-local sense\footnote{This is merely for computational convenience, as for fixed neighborhoods $B_r(x_i)$, the number of neighbors and therefore, the computational cost increases drastically with increased sampling density. While properties like locality might not hold anymore in a strict mathematical sense, this has no practical impact.}. Mathematical locality is recovered in the limit of dense sampling.
Pseudocode for this setting, including graph construction, forward pass and loss calculation is given in Algorithm \ref{alg: reg NOFE}.

\begin{algorithm}[h!] 
\caption{NOFE for predictions with point-to-point correspondence.}
\label{alg: reg NOFE}
\begin{algorithmic}[1]
\Require $D=\{(x_i,y_i)\}_{i=1}^N$, neighbors $k$, iterations $T$
\Ensure Embeddings $z_i=g(x_i)$
\State \textbf{Graph Construction:}
\State $X\gets\{x_i\}_{i=1}^N$, $E\gets\{(i,j):j\in N(i)\}$ ($N(i)$: $k$NN of $x_i$)
\State $e_{ij}\gets\{x_i,x_i-x_j\}\quad\forall(i,j)\in E$
\State
\State \textbf{Input Affinities:}
\State $p_{ij}\gets\frac{(1+\|y_i-y_j\|^2)^{-1}}{\sum_{(i,l)\in E,l\neq j}(1+\|y_i-y_l\|^2)^{-1}}$
\State
\State \textbf{Forward Pass:}
\State $h_i^{(0)}\gets L y_i$ \Comment{Lift: $L:\mathbb{R}^{d_f}\to\mathbb{R}^{d_h}$}
\State \text{MLP denotes a multilayer perceptron.}
\For{$t=0$ \textbf{to} $T-1$}
\State $K(i,j)\gets\text{MLP}(e_{ij})$ \Comment{Kernel: $\mathbb{R}^{d_e}\to\mathbb{R}^{d_h\times d_h}$}
\State $h_i^{(t+1)}\gets\text{ReLU}\Big(W\cdot h_i^{(t)} +$
\Statex\hskip\algorithmicindent\hskip\algorithmicindent$\sum_{j\in N(i)}K(i,j)\cdot h_j^{(t)}\Big)$
\EndFor
\State $z_i\gets P\cdot h_i^{(T)}$ \Comment{Project: $P:\mathbb{R}^{d_h}\to\mathbb{R}^{d_g}$}
\State
\State \textbf{Loss:}
\State $q_{ij}\gets\frac{(1+\|z_i-z_j\|^2)^{-1}}{\sum_{(i,l)\in E,l\neq j}(1+\|z_i-z_l\|^2)^{-1}}$
\State $\mathcal{L}\gets\sum_{i\neq j}p_{ij}\log\frac{p_{ij}}{q_{ij}}$, minimize via AdamW
\end{algorithmic}
\end{algorithm}

When using NOFE to predict embeddings at locations $X_q \notin X_i$, two subgraphs $A$ and $B$ are constructed. Graph $A$ is identical to the graph described previously for $X_i = X_q$. Graph $B$ uses the query locations $X_q$ to construct nodes and edges. However, it cannot use function values as node features, as they are per definition unknown and have to be initialized differently. We choose to initialize them by linear interpolation from known function values over $X_i$. Finally, every node $i_B$ of graph  $B$ is connected to its $k$-nearest-neighbors in Graph $A$. Edges are directed, allowing flow of information only within $A$ and $B$ as well as from $A$ to $B$, not from $B$ to $A$, as only nodes of $B$ will be used for prediction and all information stored in node features of $B$, originated from $A$. A corresponding pseudocode is given in Algorithm \ref{alg: ext NOFE}.

\begin{algorithm}[h!]
\caption{NOFE for predictions at arbitrary query locations.}
\label{alg: ext NOFE}
\begin{algorithmic}[1] 
\Require Source $D_s=\{(x_i^s,y_i^s)\}_{i=1}^M$, target $X_t=\{x_j^t\}_{j=1}^K$, $k$, $k_{cross}$, $T$
\Ensure Target embeddings $z_j^t=g(x_j^t)$
\State \textbf{Graph Construction:}
\State $E_s\gets\{(i,j):j\in N_s(i)\}$, $E_t\gets\{(p,q):q\in N_t(p)\}$
\State $E_{cross}\gets\{(p,i):i\in N_{cross}(p)\}$
\State Interpolate $y_p^t\gets\text{IDW}(Y_s,X_s,x_p^t,E_{cross})$
\State
\State \textbf{Input Affinities (Source):}
\State $p_{ij}^s\gets\frac{(1+\|y_i^s-y_j^s\|^2)^{-1}}{\sum_{(i,l)\in E_s,l\neq j}(1+\|y_i^s-y_l^s\|^2)^{-1}}$
\State
\State \textbf{Forward Pass:}
\State $h_i^{s(0)}\gets L_s y_i^s$, $h_p^{t(0)}\gets L_t y_p^t$
\State \text{MLP denotes a multilayer perceptron.}
\For{$t=0$ \textbf{to} $T-1$}
\State \Comment{Source intra-message passing}
\State $K_s(i,j)\gets\text{MLP}(e_{ij}^s)$
\State $h_i^{s(t+1)}\gets\text{ReLU}\Big(W_s\cdot h_i^{s(t)} +$
\Statex\hskip\algorithmicindent\hskip\algorithmicindent$\sum_{j\in N_s(i)}K_s(i,j)\cdot h_j^{s(t)}\Big)$
\State \Comment{Target intra-message passing}
\State $K_t(p,q)\gets\text{MLP}(e_{pq}^t)$
\State $h_p^{t(t+1)}\gets\text{ReLU}\Big(W_t\cdot h_p^{t(t)} +$
\Statex\hskip\algorithmicindent\hskip\algorithmicindent$\sum_{q\in N_t(p)}K_t(p,q)\cdot h_q^{t(t)}\Big)$
\State \Comment{Cross-message passing (source $\to$ target only)}
\State $K_{cross}(p,i)\gets\text{MLP}(e_{pi}^{cross})$
\State $h_p^{t(t+1)}\gets\text{ReLU}\Big(h_p^{t(t+1)} +$
\Statex\hskip\algorithmicindent\hskip\algorithmicindent$\sum_{i\in N_{cross}(p)}K_{cross}(p,i)\cdot h_i^{s(t+1)}\Big)$
\EndFor
\State $z_p^t\gets P\cdot h_p^{t(T)}$ \Comment{Target-only output}
\State
\State \textbf{Loss:}
\State $q_{pq}^t\gets\frac{(1+\|z_p^t-z_q^t\|^2)^{-1}}{\sum_{(p,r)\in E_t,r\neq q}(1+\|z_p^t-z_r^t\|^2)^{-1}}$
\State $\mathcal{L}\gets\sum_{p\neq q}p_{pq}^t\log\frac{p_{pq}^t}{q_{pq}^t}$, minimize via AdamW
\end{algorithmic}
\end{algorithm}

\paragraph{Loss Function}
Similar to t-SNE we use the discrete Kullback-Leibler divergence (Eq. \ref{eq: KL divergence}) as loss function based on pairwise affinity calculations. However, unlike t-SNE, we use the student-t kernel (Eq. \ref{eq: affinity}) for both, high- and low-dimensional space, rather than a Radial-Basis-Function (RBF) for high-dimensional data. In t-SNE the choice of kernel functions is motivated mainly by reasons of computational optimization. Further t-SNE uses a perplexity approach to estimate per point bandwidths $\sigma$, defining the RBF kernel based on local neighborhoods in feature space. While a RBF kernel in our approach is thinkable, the choice of bandwidth $\sigma$ is non-trivial and would require additional research. Further, we have experienced that using an RBF kernel for the high-dimensional space is not effective as affinities often drop to zero even for spatial neighbors. Thus, these pairs do not contribute to the loss calculation and parameter optimization.

\paragraph{Model Parameters and Training.}
NOFE is mainly determined by the parameters shown in Table \ref{tab: model parameters}.
Choices given in the table refer to the model used in the experimental part on ERA5 data (Section \ref{sec: experiments}).
This corresponds to the setup of Model 2 in the later discussed ablation study.
The final model as well as all models in the ablation study have been trained with an initial learning rate of 0.00001; a learning rate scheduler (applying a factor of 0.5 every 10 epochs on the learning rate); and AdamW optimizer, over 25 epochs.

\begin{table}[t]
\caption{Model parameters.}
\begin{center}
\begin{tabular}{l c p{8cm}}
\toprule
 Parameter & Choice  &  Explanation\\
\midrule
$W$ 	& 64 	& Dimension of lifting layer $L$; defines dimension of nodes in latent space during message passing.\\
$KW$ 	& 16 	& Width of MLP's hidden layers for prediction of kernel matrix $K(i,j)$.\\
$KD$ 	& 2 		& Number of MLP's hidden layers.\\
$k$ 	& 5 		& Number of nearest neighbors used for message passing.\\
$T$ 	& 3 		& Number of message passing iterations.\\
$P$ 	& 3 		& Width of projection layer $P$; determines dimension of embedding space.\\
\bottomrule
\end{tabular}
\end{center}
\label{tab: model parameters}
\end{table}

\paragraph{Ablation Study.}
The set of model parameters given in Table \ref{tab: model parameters} leaves us with many possible combinations of architecture setup. For this ablation study we keep the number of layers $KD=2$ fixed, as we expect this to be sufficient deep. We also keep $k=5$ constant as the computational demand is very sensitive to increasing number of neighbors and the receptive field of a given node, i.e. the effective neighborhood size of a node also increases with number of message passing steps, which we include in the ablation study. For each of the remaining parameters, we select one smaller and one larger value and combine all permutations into the eight model configurations listed in Table \ref{tab: ablation}. All models have been trained on a NVIDIA GeForce RTX 3090 GPU.

\begin{table}[t]
\caption{Parameter sweep.}
\begin{center}
\begin{tabular}{lcccrccc}
\toprule
Model & $W$  &$KW$ &  $T$ & Parameters &Training Loss &Validation Loss& Training (min.)\\
\midrule
Model 1  & 16 & 16 & 3 & 5,267 & 44.608 & 38.350 & 20 \\
Model 2  & 64 & 16 & 3 & 75,155 & 39.347 & \underline{33.026} & 44 \\
Model 3  & 16 & 64 & 3 & 21,827 & 44.418 & 38.127 & 32 \\
Model 4  & 64 & 64 & 3 & 276,035 & \underline{39.232} & 33.471 & 56  \\
Model 5  & 16 & 16 & 6 & 5,267 & 44.972 & 38.802 & 36 \\
Model 6  & 64 & 16 & 6 & 75,155 & 40.305 & 33.990 & 81 \\
Model 7  & 16 & 64 & 6 & 21,827 & 44.651 & 38.539 & 30 \\
Model 8  & 64 & 64 & 6 & 276,035 & \textbf{37.833} & \textbf{32.078}  & 108  \\
\bottomrule
\end{tabular}
\label{tab: ablation}
\end{center}
*\textbf{Bold} indicates the best result for each metric and \underline{underlining} indicates the second-best.
\end{table}

Based on the final training and validation loss values recorded in Table \ref{tab: ablation} it seems that the lifting dimension $W$ has by far the largest impact, as models with $W=64$, i.e. models with an even number, show generally lower losses. Remaining parameters seem to have negligible effects on the loss. Large number of message passing steps $T$ seems to have even negative effects on the loss minimization, even though the best performance is shown by Model 8.

However, the loss value is not the only metric to evaluate the model performance, so we repeated the experiments from section \ref{sec: experiments} on a fixed number of 5000 points per sample. Results for patch stitching errors, distance preservation metrics and gluing errors are listed in Table \ref{tab: ablation metrics}. It also includes the Pearson correlation coefficient between features $y_i$ in high-dimensional and $z_i$ embedding-space. 

Table \ref{tab: ablation metrics} shows very consistent results across all models for all metrics, with almost every model performing best at one of the metrics. Therefore, it seems reasonable to choose one of the lighter variants for experiments. In this case we selected Model 2.

\begin{table}[th]
\caption{Feature distance preservation compared across model configurations. \textbf{Bold} indicates the best result for each metric and \underline{underlining} indicates the second-best.}
\label{tab: ablation metrics}
\begin{center}
\resizebox{\textwidth}{!}{
\begin{tabular}{lcccccc}
\toprule
  \multicolumn{1}{c}{Model} &\multicolumn{1}{c}{Pairwise Corr.$\uparrow$}  &\multicolumn{1}{c}{Stress-1$\downarrow$} & \multicolumn{1}{c}{Stress-local$\downarrow$} &\multicolumn{1}{c}{$SE$-region$\downarrow$}  &\multicolumn{1}{c}{$SE$-local$\downarrow$} & \multicolumn{1}{c}{Gluing MSE$\downarrow$}
\\
\midrule
Model 1  &  0.779 $\pm$ 0.036 &             0.387 $\pm$ 0.016 &             0.093 $\pm$ 0.107 			& \underline{0.829}  $\pm$ 0.042 & \textbf{21.976}     $\pm$ 1.398 &	\underline{0.069} $\pm$ 0.041\\
Model 2  & \underline{0.790} $\pm$ 0.028 & \underline{0.380} $\pm$ 0.013 &  \textbf{0.013} $\pm$ 0.147 	& \underline{0.829}  $\pm$ 0.044 & 21.992              $\pm$ 1.345 &  $0.080 \pm 0.047$\\
Model 3  & 0.784 $\pm$ 0.035 &              0.390 $\pm$ 0.016 &             0.116 $\pm$ 0.104 			& \underline{0.829}  $\pm$ 0.042 & \textbf{21.976}     $\pm$ 1.398&  $0.076 \pm 0.047$\\
Model 4  & 0.789 $\pm$ 0.031 &      \textbf{0.379} $\pm$ 0.013 &            0.101 $\pm$ 0.107 			& 0.830              $\pm$ 0.044 & 22.129              $\pm$ 1.376 &  $0.082 \pm 0.048$\\
Model 5  & 0.780 $\pm$ 0.037 &              0.389 $\pm$ 0.016 &             0.127 $\pm$ 0.163 			& 0.830              $\pm$ 0.044 & 22.057              $\pm$ 1.335&  $0.074 \pm 0.046$\\
Model 6  &  0.789 $\pm$ 0.032 &             0.383 $\pm$ 0.014 &             0.077 $\pm$ 0.071 			& \underline{0.829}  $\pm$ 0.043 & \underline{21.979}  $\pm$ 1.309&  $0.076 \pm 0.042$\\
Model 7  & 0.783 $\pm$ 0.032 &              0.391 $\pm$ 0.013 &             0.157 $\pm$ 0.157 			& \textbf{0.828}     $\pm$ 0.042 & 22.021              $\pm$ 1.317&  \textbf{0.067} $\pm$ 0.038\\
Model 8  & \textbf{0.791} $\pm$ 0.031 &     \underline{0.380} $\pm$ 0.014 &  \underline{0.073} $\pm$ 0.097 & \underline{0.829}  $\pm$ 0.042 & 22.065              $\pm$ 1.410& $0.082 \pm 0.049$\\
\bottomrule

\end{tabular}
}
\label{tab: ablation preservation and gluing}
\end{center}
\end{table}

\section{Super-Resolution} \label{app: Superresolution}

\begin{table}
  \caption{Comparison of global (Stress-1) and local (Stress-local) feature distance preservation metrics for super-resolution experiment across methods and interpolation schemes. IDW, RBF and GP interpolation methods are applied on high-dimension, low sampling density input data onto high density grid subsequently projected to lower dimension. \textbf{Bold} indicates the best result for each metric and \underline{underlining} indicates the second-best.}
  \label{tab: super-resolution interpolation methods}
  \begin{center}
  \resizebox{\textwidth}{!}{
  \begin{tabular}{llccccccc}
  \toprule
     Metric & Interpol. method & NOFE & PCA & t-SNE & UMAP & PHATE & CONCORD & SpatialPCA \\
    \midrule
    \multirow{3}{*}{\makecell[l]{Stress-1$\downarrow$}} 
    & IDW & \underline{0.485} $\pm$ 0.028 & \textbf{0.417} $\pm$ 0.027 & 2.553 $\pm$ 0.211 & 1.183 $\pm$ 0.188 & 0.992 $\pm$ 0.001 & 0.697 $\pm$ 0.182 & 0.865 $\pm$ 0.028 \\
    & RBF & \textbf{0.512} $\pm$ 0.109 & \underline{0.594} $\pm$ 0.187 & 2.898 $\pm$ 0.256 & 0.648 $\pm$ 0.111 & 0.992 $\pm$ 0.001 & 0.857 $\pm$ 0.175 & 0.874 $\pm$ 0.030 \\
    & GP & \underline{0.499} $\pm$ 0.028 & \textbf{0.428} $\pm$ 0.028 & 2.640 $\pm$ 0.214 & 1.418 $\pm$ 0.219 & 0.992 $\pm$ 0.001 & 0.590 $\pm$ 0.137 & 0.855 $\pm$ 0.022 \\
    \midrule
    \multirow{3}{*}{\makecell[l]{Stress-local$\downarrow$}} 
    & IDW & \textbf{0.663} $\pm$ 0.027 & \underline{0.821} $\pm$ 0.017 & 0.893 $\pm$ 0.098 & 0.918 $\pm$ 0.043 & 0.999 $\pm$ 0.001 & 0.890 $\pm$ 0.161 & 0.966 $\pm$ 0.010 \\
    & RBF & \textbf{0.742} $\pm$ 0.162 & 1.129 $\pm$ 0.357 & 2.190 $\pm$ 0.238 & \underline{0.885} $\pm$ 0.044 & 0.997 $\pm$ 0.001 & 1.171 $\pm$ 0.223 & 0.939 $\pm$ 0.015 \\
    & GP & \textbf{0.677} $\pm$ 0.026 & 0.830 $\pm$ 0.015 & 0.955 $\pm$ 0.092 & 1.026 $\pm$ 0.151 & 0.999 $\pm$ 0.001 & \underline{0.789} $\pm$ 0.116 & 0.965 $\pm$ 0.010 \\
	\bottomrule
  \end{tabular}
  }
  \end{center}
\end{table}

\begin{figure}[ht]
    \centering

    \begin{tabular}{
        >{\centering\arraybackslash}m{0.01\linewidth}
        >{\centering\arraybackslash}m{0.2\linewidth}
        >{\centering\arraybackslash}m{0.2\linewidth}
        >{\centering\arraybackslash}m{0.2\linewidth}
        >{\centering\arraybackslash}m{0.2\linewidth}
    }

        & $N_i=200$ & $N_i=400$ & $N_i=1000$ & $N_i=5000$ \\

        \rotatebox[origin=c]{90}{NOFE} &
        \includegraphics[width=\linewidth]{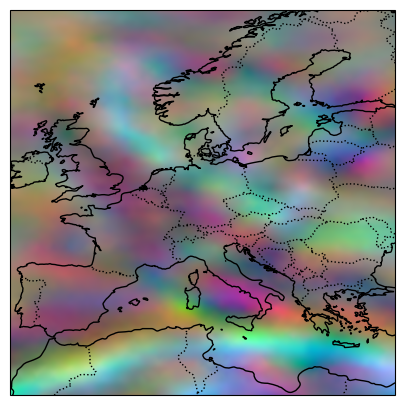} &
        \includegraphics[width=\linewidth]{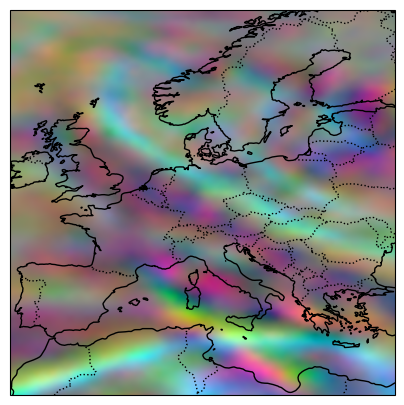} &
        \includegraphics[width=\linewidth]{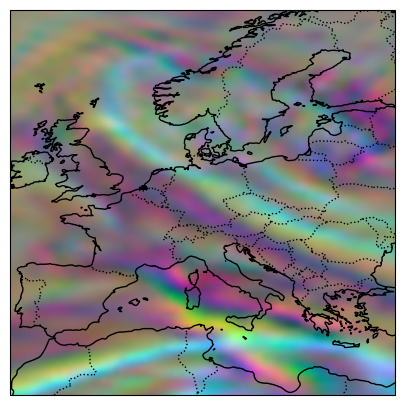} &
        \includegraphics[width=\linewidth]{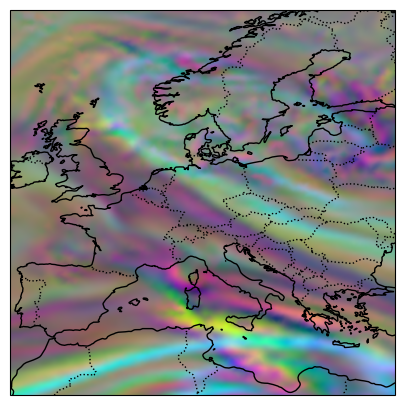}
        \\

        \rotatebox[origin=c]{90}{PCA} &
        \includegraphics[width=\linewidth]{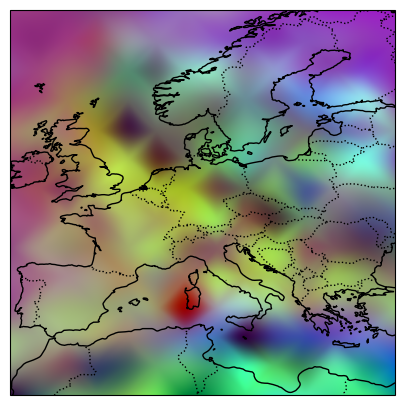} &
        \includegraphics[width=\linewidth]{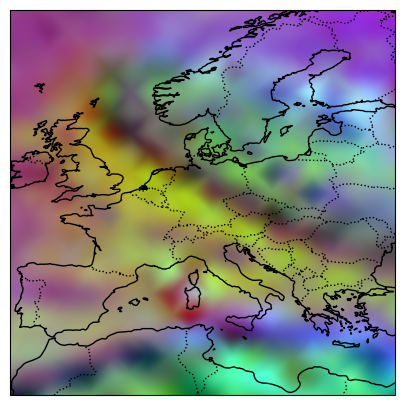} &
        \includegraphics[width=\linewidth]{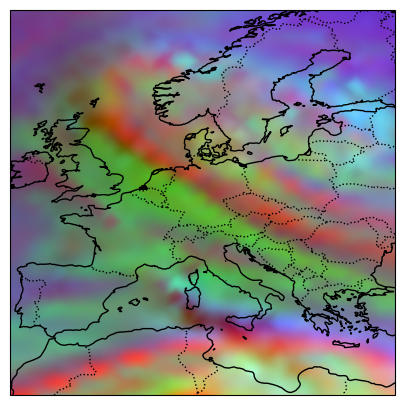} &
        \includegraphics[width=\linewidth]{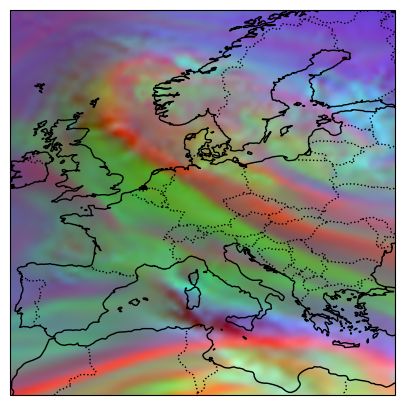}
        \\

        \rotatebox[origin=c]{90}{t-SNE} &
        \includegraphics[width=\linewidth]{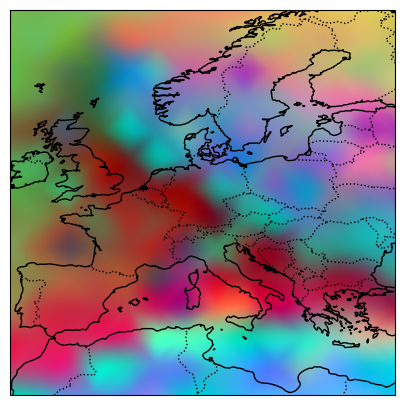} &
        \includegraphics[width=\linewidth]{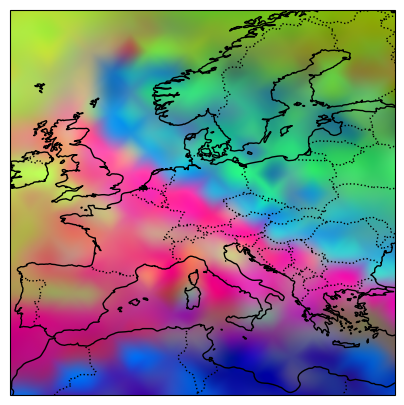} &
        \includegraphics[width=\linewidth]{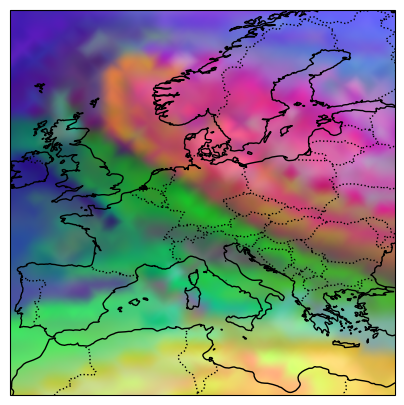} &
        \includegraphics[width=\linewidth]{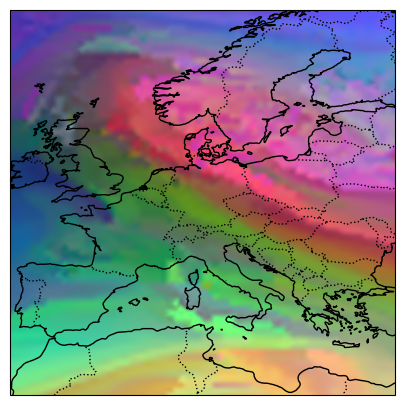}
        \\

        \rotatebox[origin=c]{90}{UMAP} &
        \includegraphics[width=\linewidth]{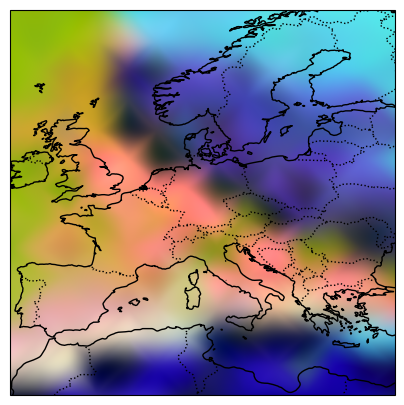} &
        \includegraphics[width=\linewidth]{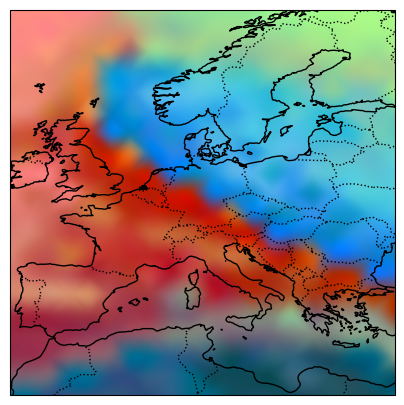} &
        \includegraphics[width=\linewidth]{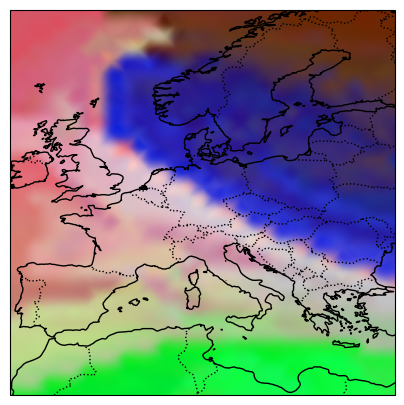} &
        \includegraphics[width=\linewidth]{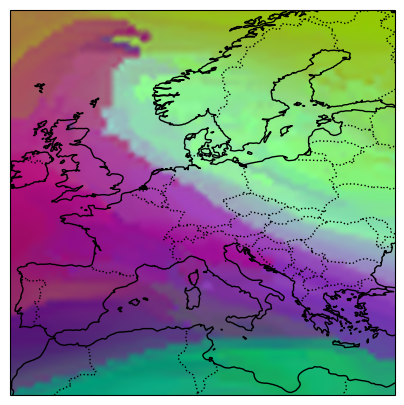}
        \\

    \end{tabular}

    \caption{Resolution comparison across methods for increasing number of input points $N$.
    While other methods can only embed the $N_i$ points provided to them, NOFE directly maps to increased resolution of $N_q=5000$ points. All results are interpolated on images with the same resolution ($500 \time 500$ pixels).
    }
    \label{fig: superresolution_matrix}
\end{figure}

As demonstrated in Section \ref{sec: superresolution}, NOFE is capable of predicting embeddings for arbitrary query locations, while traditional methods require interpolation of the results to make any predictions. While we showed results for NOFE and PCA using sample of 400 points as inputs and compared them to results of PCA with access to data at the query locations, we could see that NOFE's prediction came much closer to the structures of the high-resolution PCA. 

Here, we want to showcase more results across different methods for increasing resolution of input samples.
For this purpose, we sample for datasets of $N_i = \{200, 400, 1000, 5000\}$ points.
NOFE predicts embeddings for $N_q=5000$ locations for all samples, while remaining methods can only find embedding for input locations. Results for all methods are interpolated on a regular grid ($500 \time 500$) and visualized as images as shown in Figure \ref{fig: superresolution_matrix}.

Apart from the misalignment of embedding dimensions (here seen as inconsistency of color schemes between embeddings of the same method), traditional methods PCA, t-SNE and UMAP clearly show interpolation artifacts for small $N_i$, leading to unnatural looking patterns in the resulting data structure. This is particularly evident for $N_i = \{200, 400\}$ across t-SNE and UMAP but also seen in PCA results. With increasing $N_i$, the results start looking more natural and reveal finer details, while interpolation artifacts vanish.
NOFE on the other hand yields very consistent structures for all input resolutions. Even though lower resolution inputs lead to slightly blurred outputs without finer details, the overall structure closely resembles the results for larger $N_i$. While structures in the images get more sharper for increasing number of sample points, the patterns in the embedding look much more natural for NOFE for small $N_i$ compared to traditional methods. 

In addition, we setup a different version of this experiment, to align the interpolation schemes of NOFE with comparison methods. We use three different interpolation schemes: Inverse Distance Weighted (IDW), Radial Basis Function (RBF) and Gaussian Process (GP) interpolation. We initialize $N_i=200$ datapoints per date and interpolate these high-dimensional datapoints onto 1000 new locations. These interpolated values are used to initializing the output nodes of NOFE and given as input to the comparison methods. We evaluate the embedding quality with Stress-1 and Stress-local as before. Results are summarized in Table \ref{tab: super-resolution interpolation methods}. Overall, we see that in most cases the choice of interpolation method does not have a large impact on the stress values compared to the choice of embedding method. NOFE performs best on across all interpolation methods on the Stress-local and best or second best to PCA across all interpolation methods for the global Stress-1 metric. Thus, we conclude that even when we align NOFE and comparison methods to initialize new points with the same interpolation scheme, NOFE's super-resolution embeddings have a higher quality than comparison methods.

\section{Additional Datasets} \label{app: datasets}

We have introduced NOFE as a domain-aware, function-to-function dimensionality reduction method. Thus, the applicability of NOFE remains restricted to datasets where each datapoint is associated with a point on a continuous domain. In this section we will demonstrate, however, that NOFE is not only useful for climate data, but the application generalizes to any type of data that is defined over a continuous domain. We choose to include a dataset from medical imagery and a dataset from Earth observation with satellite imagery. We include metrics from the main manuscript but reduce the comparison methods to PCA, t-SNE and UMAP.

\begin{table}[t]
  \centering
  \caption{Summary of quantitative results to evaluate the performance of dimensionality reduction methods across different experiments. \textbf{Bold} indicates the best result for each metric and \underline{underlining} indicates the second-best.}
  \label{tab: additional datasets}
  \resizebox{\textwidth}{!}{
  \begin{tabular}{lllcccc}
    \toprule
    Data & Setting & Metric & NOFE & PCA & t-SNE & UMAP \\
    \midrule
    \multirow{5}{*}{\makecell[l]{dMRI}}
    	&\multirow{2}{*}{\makecell[l]{Embedding\\Quality}}
      	& Stress-1$\downarrow$ & \underline{0.481} $\pm$ 0.062 & \textbf{0.253} $\pm$ 0.030 & 6.167 $\pm$ 4.948 & 1.049 $\pm$ 0.155\\
      	&& Stress-local$\downarrow$ & \underline{0.500} $\pm$ 0.060 & \textbf{0.310} $\pm$ 0.033 & 6.221 $\pm$ 5.204 & 1.070 $\pm$ 0.141\\

    	&\multirow{2}{*}{\makecell[l]{Patch\\stitching}}
      	& $SE$-local$\downarrow$ & \underline{2.656} $\pm$ 4.212 & \textbf{2.490} $\pm$ 4.155 & 2.805 $\pm$ 4.743 & 3.567 $\pm$ 6.356\\
      	&& $SE$-region$\downarrow$ & \textbf{0.765} $\pm$ 0.619 & \underline{0.797} $\pm$ 0.664 & 0.844 $\pm$ 0.470 & 0.805 $\pm$ 0.592\\
      	      	
      	& Gluing & MSE (unaligned)$\downarrow$ & \textbf{0.003} $\pm$ 0.002 & \underline{0.247} $\pm$ 0.494 & 1.530 $\pm$ 2.656 & 6.620 $\pm$ 6.949 \\
      	
    \midrule
    \multirow{2}{*}{\makecell[l]{Satellite\\Images}}
       	&\multirow{2}{*}{\makecell[l]{Embedding\\Quality}}
      	& Stress-1$\downarrow$ 
      		& \underline{0.169} $\pm$ 0.045 
      		& \textbf{0.095} $\pm$ 0.028 
      		& 2.877 $\pm$ 2.251 
      		& 6.029 $\pm$ 4.403\\
      	&& Stress-local$\downarrow$ 
      		& \textbf{0.231} $\pm$ 0.183 
      		& \underline{0.299} $\pm$ 0.263 
      		& 1.809 $\pm$ 2.422 
      		& 6.128 $\pm$ 5.602\\
    \bottomrule
  	\end{tabular}
  	}%
\end{table}

\subsection{Diffusion MRI} \label{app: dMRI}

\begin{figure}[ht]
    \centering
    \includegraphics[width=0.8\linewidth]{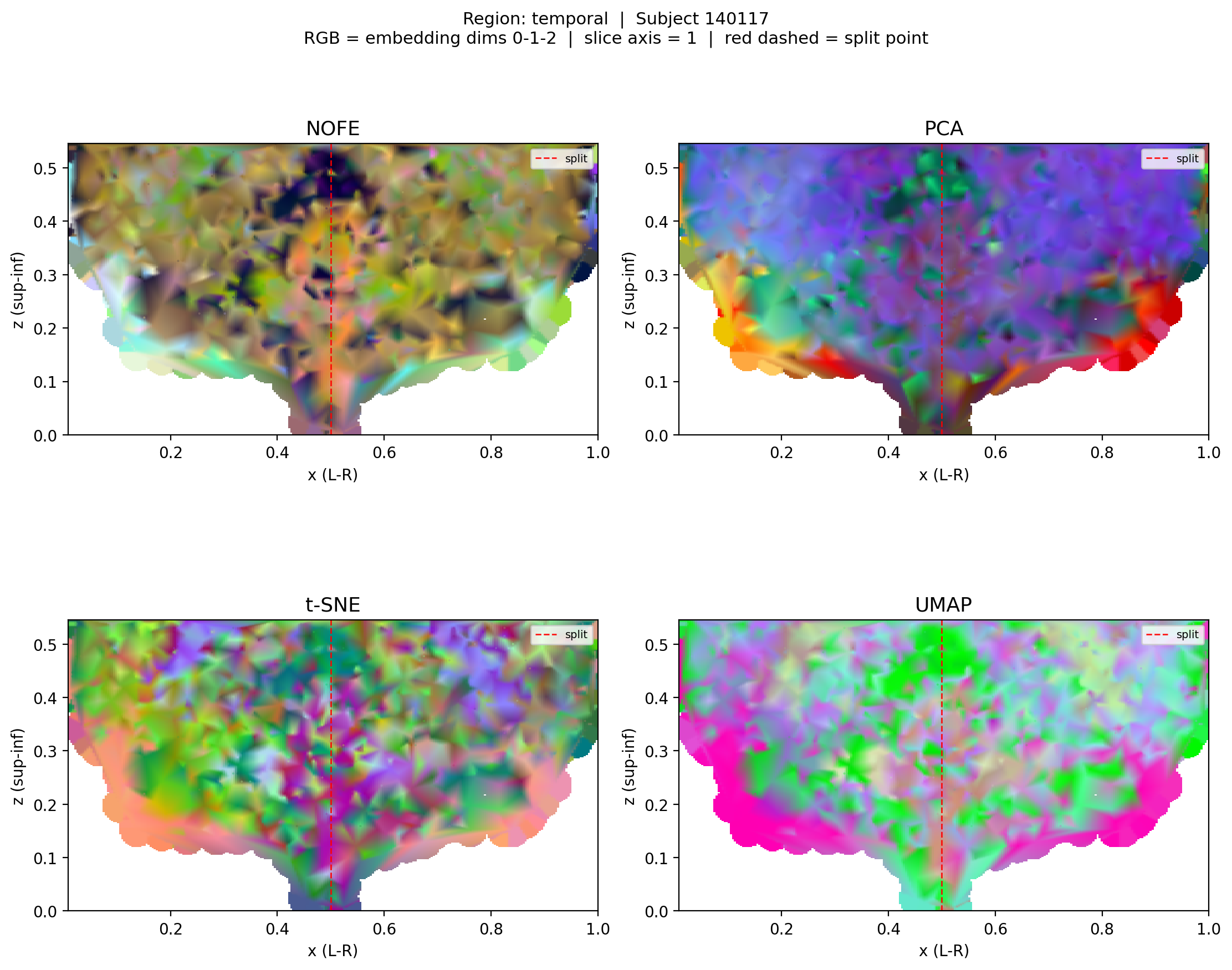}
    \caption{Patch stitching visualization in the temporal region.}
    \label{fig:patch_stitching_temporal}
\end{figure}

\begin{figure}[ht]
    \centering
    \includegraphics[width=0.8\linewidth]{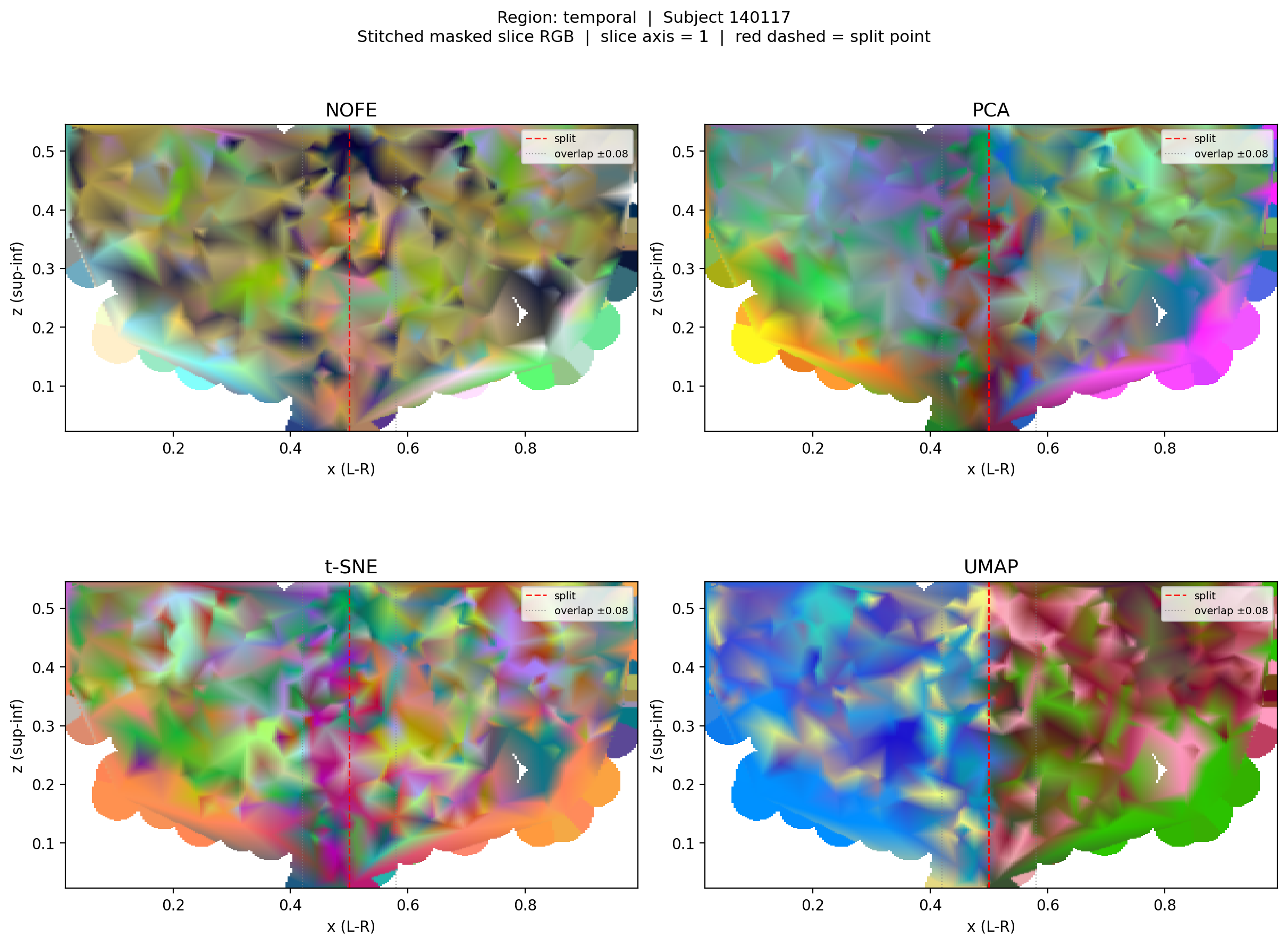}
    \caption{Patch gluing visualization in the temporal region.}
    \label{fig:patch_gluing_temporal}
\end{figure}

We evaluated NOFE on Diffusion 7T Preprocessed Recommended data from the Human Connectome Project Young Adult 2025 dataset~\citep{vanessen2012HumanConnectomeProject}. Each subject provides a brain-masked 4D diffusion MRI (Magnetic Resonance Imaging) volume, where the first three dimensions correspond to spatial voxel coordinates and the fourth dimension represents 143 measurement channels. Each of these channels represents a MRI based diffusion measurement with respect to different diffusion directions. The selected dataset contains 169 subjects. Due to storage limitations we use a subset of 16 patients to train and evaluate NOFE.

Training graphs were constructed by randomly sampling $N=10{,}000$ voxels from each subject's brain mask, with graph edges defined by the $k=10$ spatial nearest 
neighbors of each voxel. Each voxel was represented by its 143-dimensional dMRI channel signal as the node feature. For each of the selected subjects, we generated 50 training graphs and 20 validation graphs, yielding 800 training samples and 320 validation samples in total. All graph samples were stored in LMDB format. The model was trained for 100 epochs using an initial learning rate of 
$5 \times 10^{-4}$ with a StepLR scheduler using step size 5 and decay factor $\gamma=0.9$. The selected checkpoint (40,867 trainable parameters) reached a training loss of 223.26 and a validation loss of 204.78. Overall, the entire training process took around 5 hours on an Apple M4 10-core built-in GPU.

For the distance preservation experiment, we sampled 500 voxels from the frontal region of each held-out subject and embedded their 143-dimensional dMRI signals into three dimensions using NOFE, PCA, UMAP, and t-SNE. The results show that NOFE preserves the original feature-distance structure much better than UMAP and t-SNE, while PCA achieves the best scores in these metrics, as shown in Table~\ref{tab: additional datasets}.

When testing patch stitching errors, we evaluated spatial continuity by measuring whether neighboring voxels on opposite sides of a regional split remain close in the embedding space. Figure~\ref{fig:patch_stitching_temporal} and Table~\ref{tab: additional datasets} demonstrate that NOFE reaches the lowest region-normalized distance, indicating smoother embeddings across the split, while t-SNE and UMAP exhibit larger boundary discontinuities.

To evaluate the gluing property, we reported the mean normalized MSE across all selected brain regions and subjects. As shown in Table~\ref{tab: additional datasets} and Figure~\ref{fig:patch_gluing_temporal}, NOFE obtains the lowest aggregate MSE by a clear margin, suggesting that its local embeddings satisfy the gluing condition more closely than those of the baseline methods. 

\subsection{Satellite Imaging} \label{app: satellite}

\begin{figure}[ht]
\centering
\includegraphics[width=0.5\linewidth, trim=0 0 0 0]{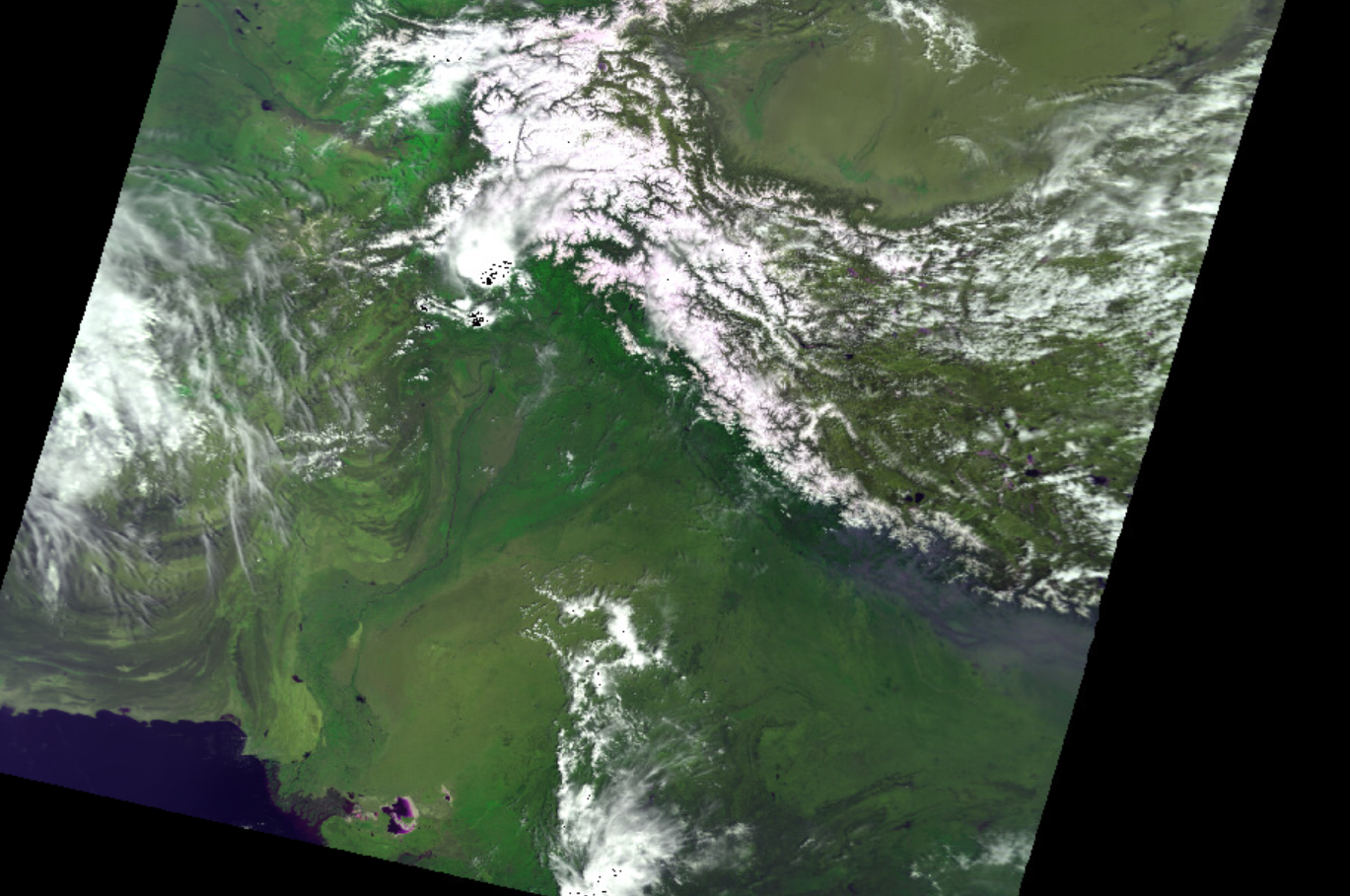}
\caption{MODIS image of the Himalaya from 2021-04-14.}
\label{fig:modis_images}
\end{figure}

We trained NOFE on a time-series of MODIS (Moderate Resolution Imaging Spectroradiometer)  data, covering the Himalaya mountain region. MODIS is a sensor onboard the Terra and Aqua satellites and consists of  36 bands ranging in wavelength from 0.4 µm to 14.4 µm, with resolutions between 250 - 1000 m. The most common application of MODIS  is tracking large-scale, daily changes in land vegetation, ocean biological activity, and atmospheric properties to understand global environmental dynamics. 

The data is processed to top-of-the-atmosphere (TOA) reflectance, and projected to WGS 84, UTM 44N coordinates. All bands are resampled to 1 km ground sampling distance (see Figure~\ref{fig:modis_images} for an example image). 

For training we use images taken in 2017. As shown in Figure~\ref{fig:modis_images}, the images often do not cover the full region. To counter this, we only consider images with roughly 40~\% coverage or more and focus on a smaller subregion in the center. This leaves us with 351 images for training. The same way we select images from January 2018 for evaluation.
\begin{figure}[h]
    \centering

    \begin{subfigure}{0.3\textwidth}
        \centering
        \includegraphics[width=\linewidth]{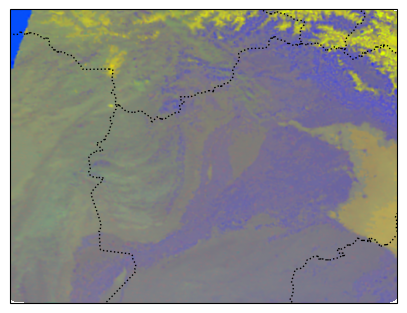}
        \caption{NOFE}
        \label{fig: modis_res nofe}
    \end{subfigure}
    \hfill
    \begin{subfigure}{0.3\textwidth}
        \centering
        \includegraphics[width=\linewidth]{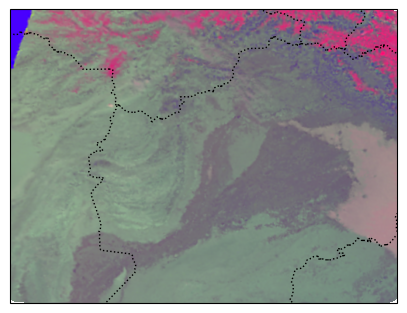}
        \caption{PCA}
        \label{fig: modis_res pca}
    \end{subfigure}
    \hfill    \begin{subfigure}{0.3\textwidth}
        \centering
        \includegraphics[width=\linewidth]{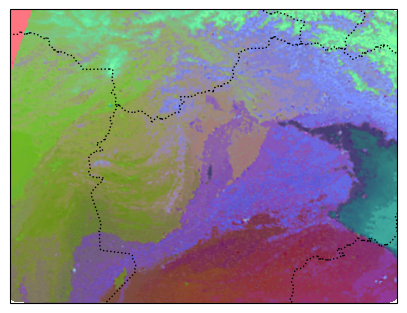}
        \caption{UMAP}
        \label{fig: modis_res umap}
    \end{subfigure}
    
    \caption{Dimensionality reduced MODIS data from 2018-01-02. $50{,}000$ points, each equipped with 36 measurements is reduced to 3-dimensional space using four different methods.}
    \label{fig: modis_res}
\end{figure}

We evaluate on feature distance preservation using previously introduced metrics and 1000 points per sample. Results in Table \ref{tab: additional datasets} are consistent with other datasets. NOFE and PCA show significantly better scores than t-SNE and UMAP. PCA leads on global metrics evaluating global structures, while NOFE has the best score for local stress.

Figure \ref{fig: modis_res} visualizes exemplary 3-dimensional embeddings from 2018-01-02 as RGB images. For the visualizations, we sampled 50,000 points. Since t-SNE computation for a dataset of this size takes up a lot of resources, we omit this method and focus on NOFE, PCA and UMAP. NOFE and PCA results seem to resolve structures very similarly, also revealing finer details, while UMAP appears to cluster larger regions and suppress smaller scaled structures.

\section{Feature Correlation} \label{app: feature correlation}
The ERA5 data provides an interesting use case since the variables carry physical meaning. This provides ``ground truth'' in the sense that we may investigate how well dimension reduction correlates with each these original variables.

One way to approach this is to use PCA on the 3-dimensional embeddings to provide maps down to grayscale values (1-dimensional space). From this, the correlation coefficients between the embeddings and the original input features are calculated. 
The correlation results are summarized in Figure \ref{fig:grayscale}.
\begin{figure}[h]
  \centering
  \includegraphics[width=\textwidth]{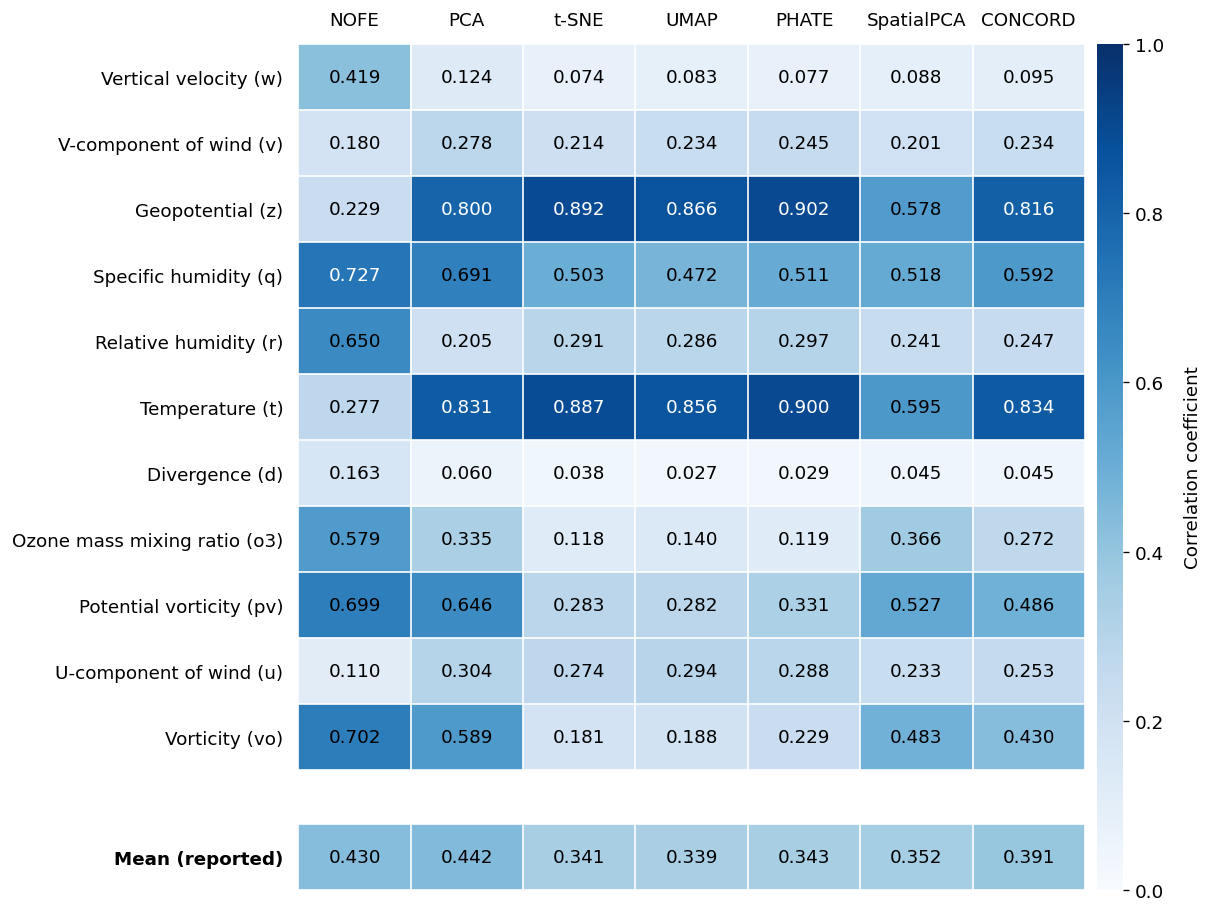}
  \caption{Visualization of correlation coefficients between grayscaled feature embeddings (i.e., 3-dimensional embeddings are projected into 1-dimensional space via PCA) with original input features.}
  \label{fig:grayscale}
\end{figure}

Studying the correlation structure shows that NOFE's correlation pattern clearly stands out compared to the six compared methods. NOFE has the strongest correlation with the vorticity and humidity values, as well as ozone mass mixing ratio and vertical velocity. All other methods have the highest correlation values with the geopotential and the temperature, as well as strong correlations with specific humidity. Notably, SpatialPCA's correlation pattern is somewhere in between NOFE and the other methods, as the correlation for geopotential and temperature are smaller and for vorticity values are higher than for other methods. PCA and NOFE have the overall highest mean correlations with 0.442 and 0.430 while the remaining methods have mean correlations between 0.339 and 0.491.

We interpret these correlations with caution but we notice that PCA, for instance, may better capture some of the variables connected to physical structure and energy \cite{IntroductionDynamicMeteorologya} (geopotential, temperature), while NOFE may reflect better variables connected to how the atmosphere is shaped and moving (tracers, e.g. humidity, ozone). This may support earlier observations, that NOFE captures more of the spatially smaller structures while other methods focus more on large scale structures like geopotential and temperature, which are observables that typically vary stronger over larger distances (e.g. temperature difference between pole and equator regions is likely larger than between areas that are just some hundred kilometers apart).

In addition to the correlation structure, we also project the 3-dimensional embeddings into 2-dimensional space (via PCA) and display their structure in form of scatterplots (see Figure \ref{fig: scatterplots-appendix}). The points are colored according to the degree of latitude of their respective location on the globe. 

The scatterplots show very different structures. NOFE, PCA and SpatialPCA show relatively little structure and seem to distribute the points more evenly. t-SNE, UMAP and PHATE all show very high degree of intrinsic structure, which corresponds to the strong clustering that we observed already in earlier plots. CONCORD is standing out as it neatly aligns all points along a curve. Interestingly, we can see that almost all patterns seem to relate strongly to the latitude. Due to overlaying points, it is not very visible for the case of CONCORD but very obvious for t-SNE, UMAP and PHATE. Also PCA shows this pattern, even though to a lower degree. SpatialPCA and NOFE have points from all latitudes distributed over the whole point cloud. That means except for SpatialPCA and NOFE, the first two principal components of the embeddings strongly reflect the latitude of the measurement location. This supports our interpretation above, that these methods focus on the large scale patterns (in this case we see a pattern from north to south), while NOFE and SpatialPCA may focus more on smaller scale patterns. This difference is likely caused by including domain information into the embeddings. These results indicate that the embeddings of remaining methods are to a large extend redundant with spatial information. That implies NOFE and SpatialPCA may encode different structural information beyond spatial location, which could be less redundant and therefore, more valuable in combination with coordinates.

\begin{figure}[htbp]
    \centering
    \includegraphics[width=\textwidth]{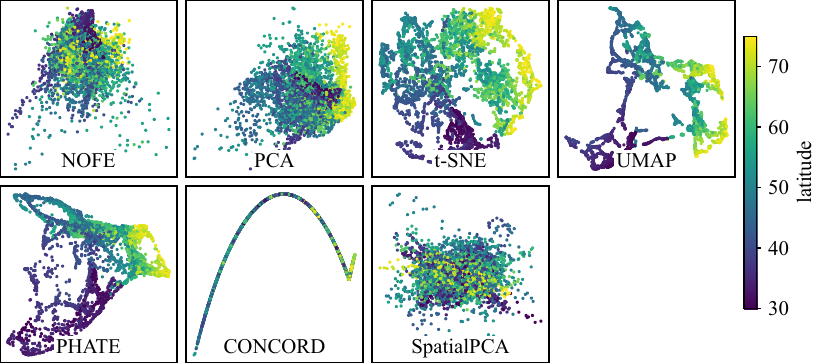}
    \caption{2-dimensional scatterplots of ERA5 embeddings across different methods. 5000 ERA5 datapoints from a region over Europe on 2019-06-15 are sampled, reduced to 3 dimensions and subsequently projected to the their two larges principle components in order to create 2-dimensional scatterplots. Color-coding indicates the latitude of the measurement location for each point.}
    \label{fig: scatterplots-appendix}
\end{figure}

\section{Sheaves, Cellular Sheaves, and Sheaf Neural Networks}
\label{app:sheaf_snn_appendix}

\subsection{Sheaves and Restriction Maps}

Let $X$ be a topological space. A presheaf $\mathcal{F}$ assigns to every open set $U \subseteq X$ a vector space $\mathcal{F}(U)$, whose elements are called sections over $U$. Intuitively, these sections represent data or signals defined on the domain $U$.

For every inclusion of open sets $V \subseteq U$, a presheaf is equipped with a linear map
\[
\rho_{V,U} : \mathcal{F}(U) \to \mathcal{F}(V),
\]
called a \emph{restriction map}. The restriction map encodes how global information is observed locally: a section defined on a larger domain $U$ induces a consistent description on any subregion $V$ by forgetting information outside of $V$.

From this perspective, restriction maps formalize the idea that data is inherently multi-scale, and that local observations are induced from global structure rather than defined independently.

A presheaf becomes a sheaf if locally consistent data uniquely determines a global object. In particular, if a collection of sections agrees on overlaps, then there exists a unique global section that restricts to them. This property enforces that local compatibility is sufficient for global reconstruction. \citep{curry2014SheavesCosheavesApplications, bredon1967SheafTheory}

\subsection{Cellular Sheaves on Graphs}

A cellular sheaf is a discretization of the above construction on a graph $G=(V,E)$. Each node $v \in V$ is equipped with a vector space $\mathcal{F}(v)$, and each edge $e \in E$ with a vector space $\mathcal{F}(e)$. In addition, for every incidence $v \trianglelefteq e$, there exists a linear restriction map
\[
\mathcal{F}_{v \trianglelefteq e} : \mathcal{F}(v) \to \mathcal{F}(e).
\]

These maps determine how node-level representations are compared in a shared edge space. In this sense, edges are not merely connectivity relations, but structured comparison spaces in which node features are projected.

For an edge $e_{ij}$ connecting nodes $v_i$ and $v_j$, consistency of the sheaf structure is expressed by the constraint
\[
\mathcal{F}_{v_i \trianglelefteq e_{ij}} x_i
=
\mathcal{F}_{v_j \trianglelefteq e_{ij}} x_j.
\]

The deviation from this constraint defines the coboundary operator
\[
(\delta x)_{e_{ij}} =
\mathcal{F}_{v_i \trianglelefteq e_{ij}} x_i
-
\mathcal{F}_{v_j \trianglelefteq e_{ij}} x_j,
\]
which measures disagreement between node representations after projection into the shared edge space.

The induced sheaf Laplacian is given by
\[
L_{\mathcal{F}} = \delta^\top \delta,
\]
and defines a diffusion operator over the sheaf. This generalizes the graph Laplacian to settings where edge-wise comparison is structured by nontrivial linear maps, and includes the connection Laplacian as a special case when restriction maps are orthogonal. \citep{hansen2020SheafNeuralNetworks, barbero2022SheafNeuralNetworks, gebhart2022GraphConvolutionalNetworks}

\subsection{Sheaf Neural Networks}

Sheaf Neural Networks (SNNs) replace standard graph diffusion with sheaf-induced diffusion. A single layer can be written as
\[
H^{(t+1)} =
\left(I - \alpha D^{-1/2} L_{\mathcal{F}} D^{-1/2}\right) H^{(t)},
\]
where $D$ denotes the block-diagonal degree operator associated with $L_{\mathcal{F}}$.

This formulation can be rewritten in message passing form as
\[
h_i^{(t+1)} =
\sigma \left(
W^{(t)} h_i^{(t)}
+
\sum_{j \in N(i)}
W^{(t)}
\mathcal{F}_{v_i \trianglelefteq e_{ij}}^\top
\mathcal{F}_{v_j \trianglelefteq e_{ij}}
h_j^{(t)}
\right).
\]

In this expression, 
\[
\phi(e_{ij}) =
\mathcal{F}_{v_i \trianglelefteq e_{ij}}^\top
\mathcal{F}_{v_j \trianglelefteq e_{ij}}
\]
acts as an edge-induced consistency operator that determines how information is transferred between nodes after being lifted into the edge space.

From this viewpoint, SNNs replace the unconstrained aggregation mechanism of standard graph neural networks with a structured diffusion process governed by local compatibility constraints. The restriction maps determine the geometry of comparison between neighboring node features, and thereby induce a notion of coherence that is stronger than adjacency alone. \citep{hansen2020SheafNeuralNetworks, barbero2022SheafNeuralNetworks, gebhart2022GraphConvolutionalNetworks, bodnar2022NeuralSheafDiffusion}

\subsection{Interpretation}

SNNs can thus be understood as diffusion processes on structured bundles over graphs, where edge spaces define local coordinate systems for comparison of node features. The restriction maps determine how information is transported into these local coordinates, and the Laplacian measures inconsistency of these transported representations. \citep{bodnar2022NeuralSheafDiffusion}

This perspective makes explicit that the expressive power of SNNs does not only arise from message passing, but from the geometry induced by the restriction maps, which define how local representations are embedded into shared comparison spaces.

\subsection{Connection to Vector Diffusion Maps}

Sheaf-based diffusion operators are closely related to the manifold learning framework of Vector Diffusion Maps (VDM) \citep{singer2012VectorDiffusionMaps}, which extend Diffusion Maps \citep{coifman2006DiffusionMaps} to settings with vector-valued features and transport along edges. In both cases, the central object is not only a notion of connectivity, but a notion of \emph{consistent transport} of representations across the graph.

In Vector Diffusion Maps, one considers a point cloud $\{x_i\}_{i=1}^n \subset \mathbb{R}^D$ sampled from an underlying manifold $\mathcal{M}$. A weighted graph $G=(V,E)$ is constructed to approximate the local geometry of $\mathcal{M}$, and each edge is equipped with an orthogonal transformation $O_{ij} \in O(d)$ encoding a local estimate of alignment between tangent-space coordinates at $x_i$ and $x_j$.

The resulting diffusion operator models how vectors propagate consistently over the manifold geometry by first transporting neighboring contributions into a common coordinate system and then aggregating them according to scalar affinities. 

From this perspective, VDM defines a discretization of diffusion on a vector bundle over $\mathcal{M}$, where both the geometry of the manifold and the consistency of local coordinate systems are encoded through the graph structure and the edge-wise orthogonal transformations.

In particular, VDM corresponds to the rigid case where transport is fixed and orthogonal, whereas SNNs allow the transport structure itself to be learned or constrained, thereby extending diffusion-based geometric learning to more general consistency structures beyond Riemannian or orthogonal settings. \citep{bodnar2022NeuralSheafDiffusion}

This connection highlights that SNNs inherit the geometric interpretation of diffusion on bundles from Vector Diffusion Maps, while significantly generalizing the notion of admissible local structure through the flexibility of sheaf-theoretic restriction maps.

We see NOFE as a sheaf based framework, performing mappings in alignment with sheaf structure. Thus, we can adapt above interpretations of message passing over the graph as a learned vector diffusion on a vector bundle over $\mathcal{M}$. While both, NOFE and VDMs estimate this diffusion process for dimensionality reduction, they cannot be compared head-to-head due to conceptually different problem framing. VDMs are used for manifold learning; they estimate how a vector field would diffuse over a manifold, to learn the structure of the underlying manifold itself and find embeddings for its elements $x_i \in \mathcal{M}$. In contrast, NOFE aims to find embeddings of an actual vector field realization, with elements $y(x_i)$ from a vector-valued function $f : \mathcal{M} \rightarrow \mathbb{R}^{d_f}$ defined over the manifold. To this end, it performs, or simulates an actual diffusion process through message passing and finally embeds the function values of $f$ in a lower-dimensional space, represented by a new function $g: \mathcal{M} \rightarrow \mathbb{R}^{d_g}$ with $d_g < d_f$, defined over the same manifold.

\section{Computational Costs} \label{app. computational costs}

\begin{table}[h]
  \caption{Wall times of NOFE embeddings across different sample and neighborhood sizes. The table separately records preprocessing, which includes graph construction and kNN search, and the forward pass, which produces the actual embeddings. The CPU-heavy preprocessing was performed on an Apple Macbook Air (M2), while the forward passes are performed by a single NVIDIA RTX A6000 GPU.}
  \label{tab: computation_times}
  \centering
\begin{tabular}{rrr@{\;}l@{\;}c}
    \toprule
      N & k & \multicolumn{2}{c}{Preprocessing [ms]} & Forward Pass [ms] \\
    \midrule
       & 2 & 5.10 &$\pm$ 0.83 & 1.31 $\pm$ 0.03 \\
      500 & 6 & 5.78 &$\pm$ 3.98 & 1.35 $\pm$ 0.07 \\
       & 10 & 5.09 &$\pm$ 0.69 & 1.36 $\pm$ 0.08 \\
       \midrule
       & 2 & 6.03 &$\pm$ 0.69 & 1.36 $\pm$ 0.14 \\
      1000 & 6 & 6.64 &$\pm$ 0.75 & 1.37 $\pm$ 0.07 \\
       & 10 & 6.85 &$\pm$ 0.42 & 1.43 $\pm$ 0.11 \\
       \midrule
       & 2 & 22.65 &$\pm$ 1.24 & 1.47 $\pm$ 0.16 \\
      5000 & 6 & 29.59 &$\pm$ 1.03 & 1.46 $\pm$ 0.13 \\
       & 10 & 34.40 &$\pm$ 0.79 & 1.46 $\pm$ 0.08 \\
       \midrule
       & 2 & 47.58 &$\pm$ 0.70 & 1.47 $\pm$ 0.13 \\
      10000 & 6 & 68.28 &$\pm$ 0.97 & 1.51 $\pm$ 0.15 \\
       & 10 & 82.90 &$\pm$ 5.89 & 1.58 $\pm$ 0.17 \\
    \bottomrule
  \end{tabular}
\end{table}

The computational costs of a method are an important factor for its applicability. In order to give the reader an estimate of NOFE's computational costs, we run an experiment to track computation times and how they scale with different parameters. In section \ref{app: implementation}, we have tracked training time for different model-specific hyperparameters. Here, we will focus on inference times and how they scale with the number of input points $N$. We include sample sizes between $N=500$ and $N=100,000$. Additionally, we run the experiments for different numbers of neighborhood connections with $k = \{5, 10, 15\}$, as the neighborhood search is typically not scaling favorably. To give a more detailed insight, we have split up the embedding process into the preprocessing of the data, i.e., the construction of a graph, including data normalization, neighborhood search, projection from GPS to Cartesian coordinates, etc.~and the actual forward pass. For each configuration we include 10 runs and report mean and standard deviation of the computation times. Since the current implementation includes faiss-cpu neighborhood search, the preprocessing is done on an Apple MacBook Air (M2). The forward pass is run on a NVIDIA RTX A6000 GPU.

\section{Additional Visualizations} \label{app: additional visualisations}

In the following we showcase additional results for the patch-stitching and gluing experiments from section \ref{sec: experiments} for different, arbitrary dates of ERA5 data.

Figure \ref{fig: patch-stitching-appendix} shows results for the patch-stitching experiment from section \ref{sec: Consistency and Sampling Independence} for data of 2019-11-05. The results underpin our earlier assessment that NOFE embeds the disjoint datasets in a consistent way such that the plot shows a smooth transition between the different subregion, while plots of all other methods show a clear border between these regions. This indicates that independently embedded datasets here cannot be interpreted consistently.

\begin{figure}[htbp]
    \centering

    \begin{subfigure}{0.24\textwidth}
        \centering
        \includegraphics[width=\linewidth]{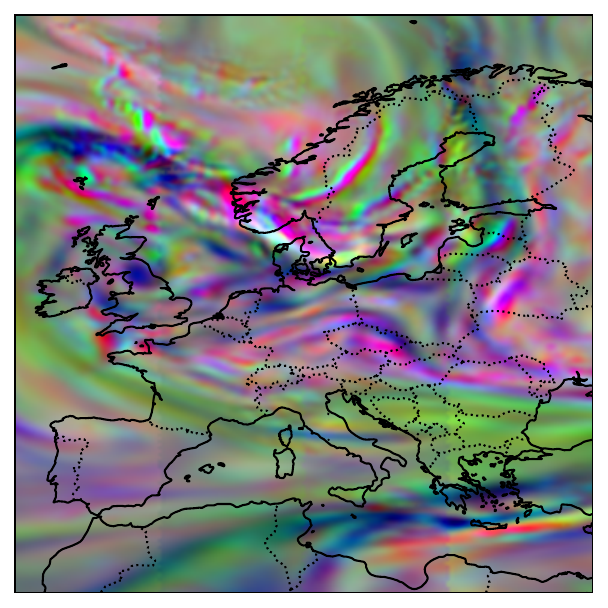}
        \caption{NOFE}
        \label{fig: patch-stitching-appendix-nofe}
    \end{subfigure}
    \hfill
    \begin{subfigure}{0.24\textwidth}
        \centering
        \includegraphics[width=\linewidth]{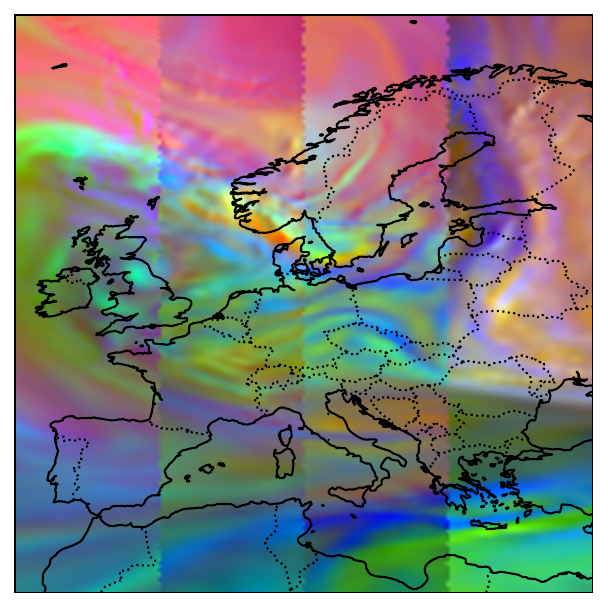}
        \caption{PCA}
        \label{fig: patch-stitching-appendix-pca}
    \end{subfigure}
    \hfill
    \begin{subfigure}{0.24\textwidth}
        \centering
        \includegraphics[width=\linewidth]{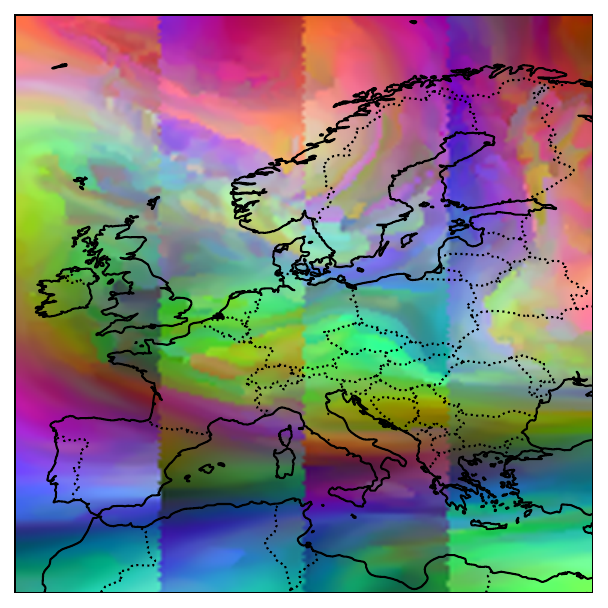}
        \caption{t-SNE}
        \label{fig: patch-stitching-appendix-tsne}
    \end{subfigure}
	\hfill
    \begin{subfigure}{0.24\textwidth}
        \centering
        \includegraphics[width=\linewidth]{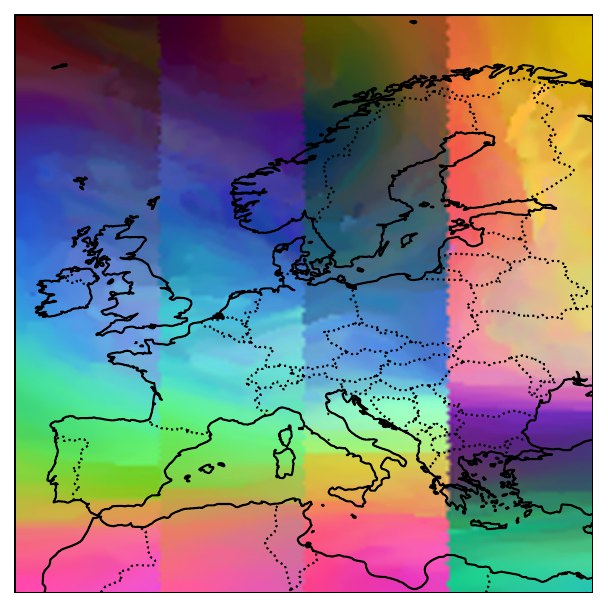}
        \caption{UMAP}
        \label{fig: patch-stitching-appendix-umap}
    \end{subfigure}

	\medskip

    \begin{subfigure}{0.24\textwidth}
        \centering
        \includegraphics[width=\linewidth]{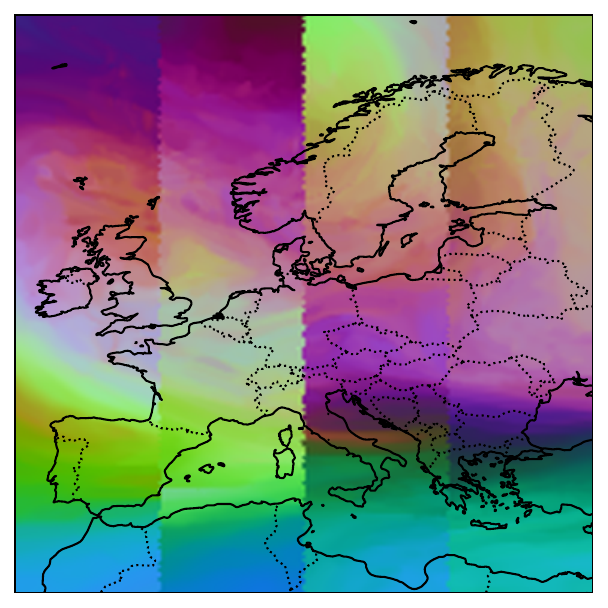}
        \caption{PHATE}
        \label{fig: patch-stitching-appendix-phate}
    \end{subfigure}
    \hfill
    \begin{subfigure}{0.24\textwidth}
        \centering
        \includegraphics[width=\linewidth]{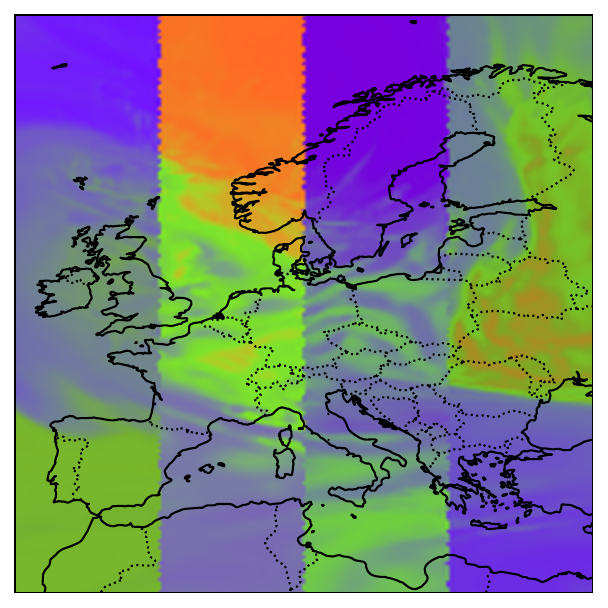}
        \caption{CONCORD}
        \label{fig: patch-stitching-appendix-concord}
    \end{subfigure}
	\hfill
    \begin{subfigure}{0.24\textwidth}
        \centering
        \includegraphics[width=\linewidth]{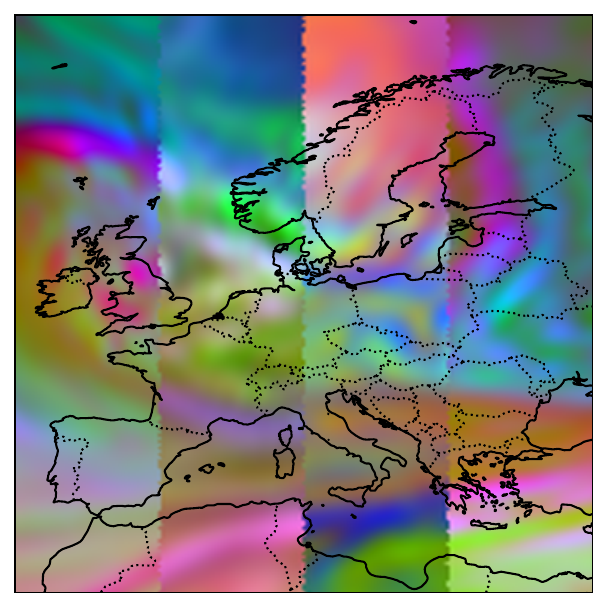}
        \caption{SpatialPCA}
        \label{fig: patch-stitching-appendix-spatialpca}
    \end{subfigure}
    \hfill
	\begin{subfigure}{0.24\textwidth}
    	\centering
    	\rule{0pt}{4cm}
	\end{subfigure}
    \caption{Patch-stitching experiment for ERA5 data from 2019-11-05. Data from the region over Europe is split up into 4 disjoint subset, which are embedded independently with a given method and subsequently merged back together. The figure shows results across different methods.}
    \label{fig: patch-stitching-appendix}
\end{figure}

Figure \ref{fig: gluing-appendix} displays the results of the gluing experiment from section \ref{sec: gluing} for data of 2019-04-17. The results support earlier observations from section \ref{sec: gluing}. For NOFE, the gluing region in which the two patches overlap is almost indistinguishable from the rest, while comparison methods show a clear misalignment of the two separate embedding, such that individual points become visible rather than showing a smooth color transition. Notably, the gluing region in the PCA, PHATE and SpatialPCA images show better alignment than the remaining methods (excluding NOFE). The visually strongest misalignment is found in the CONCORD result.
\begin{figure}[htbp]
    \centering
    \begin{subfigure}{0.24\textwidth}
        \centering
        \includegraphics[width=\linewidth]{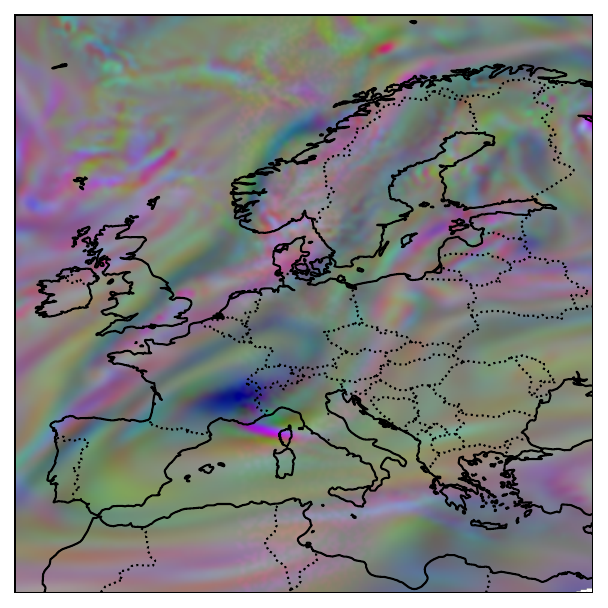}
        \caption{NOFE}
        \label{fig: gluing-appendix-nofe}
    \end{subfigure}
    \hfill
    \begin{subfigure}{0.24\textwidth}
        \centering
        \includegraphics[width=\linewidth]{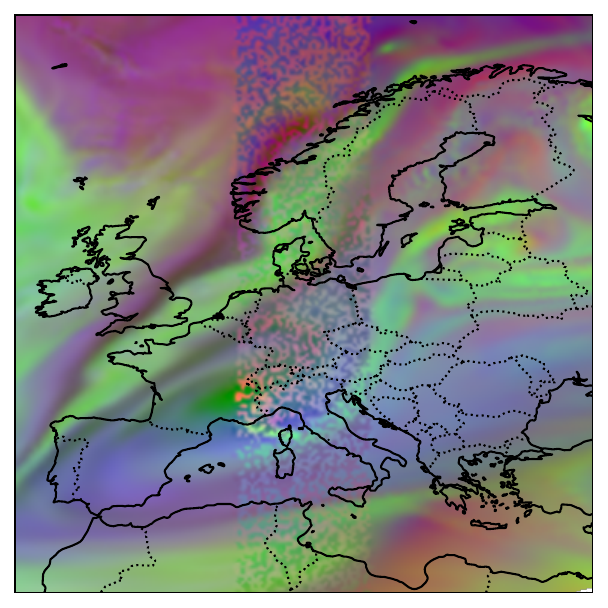}
        \caption{PCA}
        \label{fig: gluing-appendix-pca}
    \end{subfigure}
    \hfill
    \begin{subfigure}{0.24\textwidth}
        \centering
        \includegraphics[width=\linewidth]{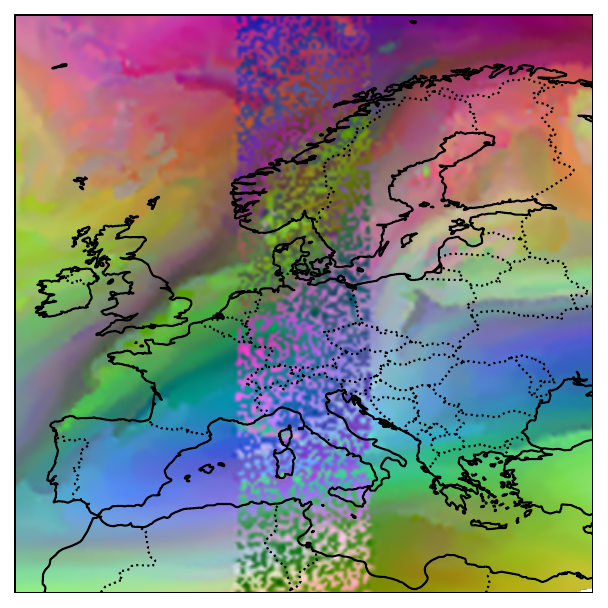}
        \caption{t-SNE}
        \label{fig: gluing-appendix-tsne}
    \end{subfigure}
	\hfill
    \begin{subfigure}{0.24\textwidth}
        \centering
        \includegraphics[width=\linewidth]{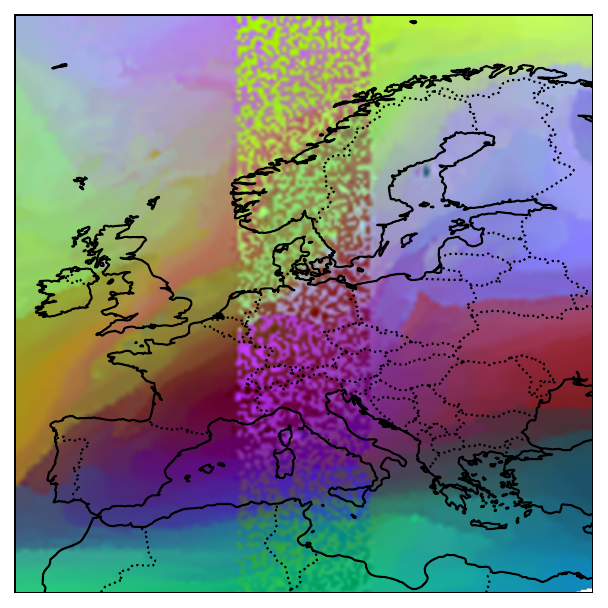}
        \caption{UMAP}
        \label{fig: gluing-appendix-umap}
    \end{subfigure}

	\medskip

    \begin{subfigure}{0.24\textwidth}
        \centering
        \includegraphics[width=\linewidth]{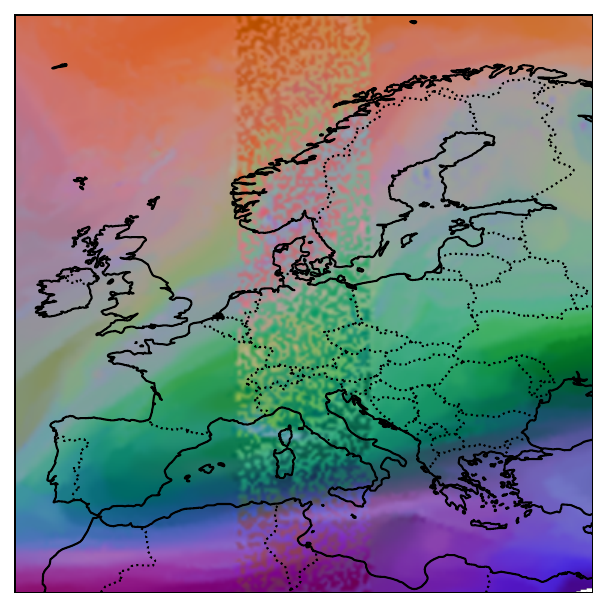}
        \caption{PHATE}
        \label{fig: gluing-appendix-phate}
    \end{subfigure}
    \hfill
    \begin{subfigure}{0.24\textwidth}
        \centering
        \includegraphics[width=\linewidth]{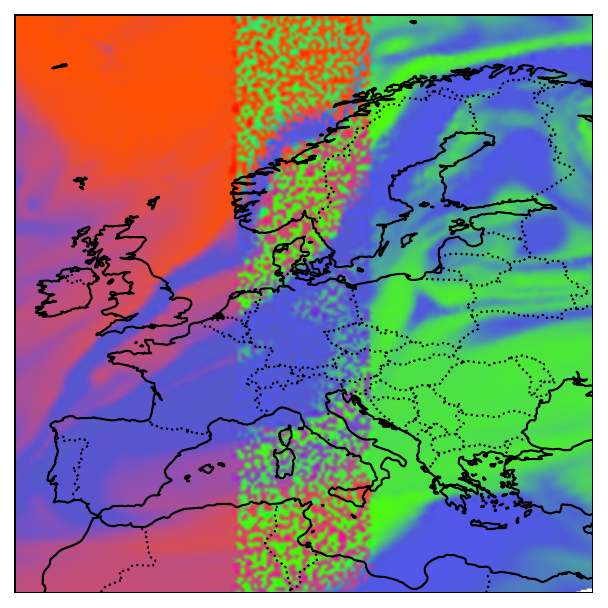}
        \caption{CONCORD}
        \label{fig: gluing-appendix-concord}
    \end{subfigure}
	\hfill
    \begin{subfigure}{0.24\textwidth}
        \centering
        \includegraphics[width=\linewidth]{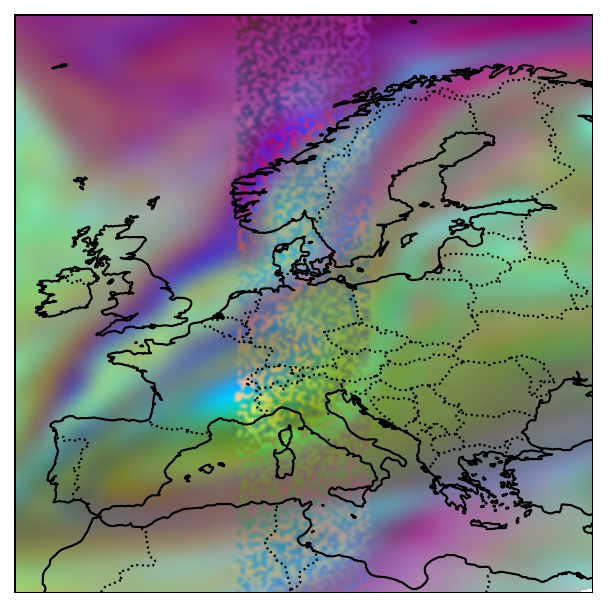}
        \caption{SpatialPCA}
        \label{fig: gluing-appendix-spatialpca}
    \end{subfigure}
    \hfill
	\begin{subfigure}{0.24\textwidth}
    	\centering
    	\rule{0pt}{4cm}
	\end{subfigure}
    \caption{Visualization of the gluing experiment for ERA5 data over Europe from 2019-04-17. Here, two separate but overlapping regions are embedded independently and subsequently merged together into a single plot. The figure shows results across different methods.}
    \label{fig: gluing-appendix}
\end{figure}

\newpage
\section*{NeurIPS Paper Checklist}

\begin{enumerate}

\item {\bf Claims}
    \item[] Question: Do the main claims made in the abstract and introduction accurately reflect the paper's contributions and scope?
    \item[] Answer: \answerYes{}{} 
    \item[] Justification: We have summarized our contributions and claims in the abstract to reflect what is written throughout the paper, including in the introduction. We have highlighted the novelty of our approach as a continuous, domain-aware method using GKOs and sheaf theory. Claims are supported by propositions 1–3 and Section 4 experiments.

\item {\bf Limitations}
    \item[] Question: Does the paper discuss the limitations of the work performed by the authors?
    \item[] Answer: \answerYes{} 
    \item[] Justification: We have included a subsection as part of the conclusion where we discuss the practical limitations of our approach in terms of computational scaling, hyperparameter tuning and how sensitive it can be; and also the need to prioritize the number of comparisons with gold standard and widely used methods given the computational cost.

\item {\bf Theory assumptions and proofs}
    \item[] Question: For each theoretical result, does the paper provide the full set of assumptions and a complete (and correct) proof?
    \item[] Answer: \answerYes{}{} 
    \item[] Justification: We have added three propositions in Section 3. Assumptions are explicitly stated (e.g., Assumption 1: functions are random global sections of a continuous sheaf; locality of the GKO operator). Proofs are provided in the main text and rely on standard sheaf restriction properties and operator locality arguments. Additional information, e.g., pseudocodes, have been added in the supplementary material and are properly referenced in the text.

    \item {\bf Experimental result reproducibility}
    \item[] Question: Does the paper fully disclose all the information needed to reproduce the main experimental results of the paper to the extent that it affects the main claims and/or conclusions of the paper (regardless of whether the code and data are provided or not)?
    \item[] Answer: \answerYes{} 
    \item[] Justification: We fully disclose all information required to reproduce the main experimental results. This includes dataset specifications, graph construction procedures, loss functions (KL divergence over domain neighbors), optimization settings, and evaluation metrics. Pseudocode for both point-to-point and arbitrary query predictions is included in Appendix A. All necessary details are provided in the main text and supplementary material.

\item {\bf Open access to data and code}
    \item[] Question: Does the paper provide open access to the data and code, with sufficient instructions to faithfully reproduce the main experimental results, as described in supplemental material?
    \item[] Answer: \answerYes{} 
    \item[] Justification: We will release the complete source code, training scripts, and evaluation pipelines under an open-source license upon acceptance. Additionally, we will provide anonymized URLs to the processed datasets (ERA5, Diffusion MRI, MODIS) and detailed instructions for data preparation, ensuring full reproducibility and compliance with NeurIPS data/code policies.

\item {\bf Experimental setting/details}
    \item[] Question: Does the paper specify all the training and test details (e.g., data splits, hyperparameters, how they were chosen, type of optimizer) necessary to understand the results?
    \item[] Answer: \answerYes{} 
    \item[] Justification: We specify all critical experimental details, including data splits (training/validation/test sets), hyperparameters (learning rate, StepLR scheduler, K, T, network dimensions), optimizer (AdamW), and evaluation metrics. These details are provided in Section 4 and Appendix A, along with supplementary tables for ablation studies and additional datasets.

\item {\bf Experiment statistical significance}
    \item[] Question: Does the paper report error bars suitably and correctly defined or other appropriate information about the statistical significance of the experiments?
    \item[] Answer: \answerYes{} 
    \item[] Justification: We report mean and standard deviation across multiple independent samples (e.g., daily predictions across January 2019 or multiple subjects in the dMRI dataset) for all quantitative metrics. Error bars represent standard deviations computed over these independent runs. We clearly state the source of variability in the text and captions and ensure that statistical reporting aligns with NeurIPS guidelines.

\item {\bf Experiments compute resources}
    \item[] Question: For each experiment, does the paper provide sufficient information on the computer resources (type of compute workers, memory, time of execution) needed to reproduce the experiments?
    \item[] Answer: \answerYes{} 
    \item[] Justification: We disclose the compute resources used for all experiments, including GPU type, memory requirements, training time per epoch, and total compute hours. These details are provided in the supplementary material and will be explicitly stated in the final version to ensure transparency regarding computational costs and reproducibility.
    
\item {\bf Code of ethics}
    \item[] Question: Does the research conducted in the paper conform, in every respect, with the NeurIPS Code of Ethics \url{https://neurips.cc/public/EthicsGuidelines}?
    \item[] Answer: \answerYes{} 
    \item[] Justification: All research conducted in this paper strictly adheres to the NeurIPS Code of Ethics. We have reviewed and complied with all ethical guidelines, including proper attribution of prior work, transparent reporting of methods and results, and responsible handling of datasets. No ethical violations or conflicts of interest exist.

\item {\bf Broader impacts}
    \item[] Question: Does the paper discuss both potential positive societal impacts and negative societal impacts of the work performed?
    \item[] Answer: \answerYes{} 
    \item[] Justification: We discuss broader impacts in the Conclusion and Supplementary Material. Positive impacts include enabling more accurate, continuous, and domain-aware dimensionality reduction for scientific data (climate, medical imaging, remote sensing), which can improve climate modeling, diagnostic accuracy, and environmental monitoring. Negative impacts or risks are minimal, as the method operates on already-collected scientific datasets and does not involve personal data or surveillance. We encourage responsible application in scientific and industrial contexts.
    
\item {\bf Safeguards}
    \item[] Question: Does the paper describe safeguards that have been put in place for responsible release of data or models that have a high risk for misuse (e.g., pre-trained language models, image generators, or scraped datasets)?
    \item[] Answer: \answerYes{} 
    \item[] Justification: The datasets used (ERA5, Human Connectome Project dMRI, MODIS satellite imagery) are publicly available, open-access scientific datasets with appropriate usage licenses. No personally identifiable information or high-risk dual-use data is involved. We will implement standard open-source release practices, including clear documentation, usage guidelines, and attribution requirements, to ensure responsible deployment.

\item {\bf Licenses for existing assets}
    \item[] Question: Are the creators or original owners of assets (e.g., code, data, models), used in the paper, properly credited and are the license and terms of use explicitly mentioned and properly respected?
    \item[] Answer: \answerYes{} 
    \item[] Justification: All external datasets and reference implementations are properly cited in the references. We explicitly state the licenses and terms of use for ERA5 (Copernicus Climate Change Service), HCP dMRI data (HCP Consortium), and MODIS data (NASA Earthdata). All third-party libraries and baselines (PCA, t-SNE, UMAP) are appropriately credited with their original publications and licenses.

\item {\bf New assets}
    \item[] Question: Are new assets introduced in the paper well documented and is the documentation provided alongside the assets?
    \item[] Answer: \answerYes{} 
    \item[] Justification: We will release a new open-source codebase implementing NOFE, complete with detailed documentation, installation instructions, example notebooks, and pre-trained checkpoints. The code will follow standard open-source licensing and include comprehensive README files, API documentation, and contribution guidelines to ensure usability and reproducibility.

\item {\bf Crowdsourcing and research with human subjects}
    \item[] Question: For crowdsourcing experiments and research with human subjects, does the paper include the full text of instructions given to participants and screenshots, if applicable, as well as details about compensation (if any)? 
    \item[] Answer: \answerNA{} 
    \item[] Justification: This research does not involve crowdsourcing, human subjects, or human-subject data collection. All experiments are conducted on publicly available scientific datasets (climate reanalysis, diffusion MRI, satellite imagery) that contain no personally identifiable information or require human participant consent.

\item {\bf Institutional review board (IRB) approvals or equivalent for research with human subjects}
    \item[] Question: Does the paper describe potential risks incurred by study participants, whether such risks were disclosed to the subjects, and whether Institutional Review Board (IRB) approvals (or an equivalent approval/review based on the requirements of your country or institution) were obtained?
    \item[] Answer: \answerNA{} 
    \item[] Justification: As this work does not involve human subjects, clinical trials, or crowdsourced labor, IRB approval or equivalent ethical review is not applicable. All data used are pre-existing, anonymized scientific datasets released for public research use.


\item {\bf Declaration of LLM usage}
    \item[] Question: Does the paper describe the usage of LLMs if it is an important, original, or non-standard component of the core methods in this research? Note that if the LLM is used only for writing, editing, or formatting purposes and does \emph{not} impact the core methodology, scientific rigor, or originality of the research, declaration is not required.
    \item[] Answer: \answerNA{} 
    \item[] Justification: LLMs were not used as a component of the core methodology, architecture, or experimental pipeline in this research. Any LLM usage was strictly limited to the Gemini service offered by NeurIPS.

\end{enumerate}
\end{document}